\documentclass{article}

     \PassOptionsToPackage{numbers, compress}{natbib}

    \usepackage[preprint]{neurips_2021}

\usepackage[utf8]{inputenc} 
\usepackage[T1]{fontenc}    
\usepackage{hyperref}       
\usepackage{url}            
\usepackage{booktabs}       
\usepackage{amsfonts}       
\usepackage{nicefrac}       
\usepackage{microtype}      
\usepackage{xcolor}         

\usepackage[utf8]{inputenc} 
\usepackage[T1]{fontenc}    
\usepackage{hyperref}       
\usepackage{url}            
\usepackage{booktabs}       
\usepackage{amsfonts}       
\usepackage{nicefrac}       
\usepackage{microtype}      
\usepackage{graphicx}
\usepackage{xcolor}

\usepackage{amsmath}
\usepackage{amsfonts}
\usepackage{amssymb}

\newtheorem{theorem}{Theorem}[section]

\newtheorem{proposition}{Proposition}[section]

\newcommand{\cA}{{\mathcal{A}}}

\newcommand{\cS}{{\mathcal{S}}}

\newcommand{\cP}{{\mathcal{P}}}
\newcommand{\cL}{{\mathcal{L}}}
\newcommand{\cX}{{\mathcal{X}}}

\newcommand{\cF}{{\mathcal{F}}}
\newcommand{\cN}{{\mathcal{N}}}
\newcommand{\cQ}{{\mathcal{Q}}}

\newcommand{\cT}{{\mathcal{T}}}

\newcommand{\cO}{{\mathcal{O}}}
\newcommand{\cH}{{\mathcal{H}}}

\newcommand{\be}{{\textbf{e}}}

\newcommand{\bP}{\textbf{P}}
\newcommand{\bQ}{\textbf{Q}}

\newcommand{\bX}{\textbf{X}}
\newcommand{\bY}{\textbf{Y}}
\newcommand{\bx}{\textbf{x}}
\newcommand{\by}{\textbf{y}}

\newcommand{\bq}{\textbf{q}}

\newcommand{\bz}{\textbf{z}}
\newcommand{\bc}{\textbf{c}}

\newcommand{\br}{\textbf{r}}
\newcommand{\bt}{\textbf{t}}
\newcommand{\bw}{\textbf{w}}

\newcommand{\bpii}{\pmb{\pi}} 
\newcommand{\bPi}{\pmb{\Pi}}

\newcommand{\bdt}{{\pmb{\delta}}}

\newcommand{\bphi}{{\pmb{\phi}}}

\newcommand{\bbR}{\mathbb{R}}
\newcommand{\bbE}{\mathbb{E}}
\newcommand{\bbP}{\mathbb{P}}

\newcommand{\UR}{\textsc{UR}}

\newenvironment{proof}[1][Proof]{\noindent\textbf{#1.} }{\ \rule{0.5em}{0.5em}}

\def\htien#1{}

\newcommand{\comments}[1]{{\color{blue}\textit{$\#$ #1}}}

\usepackage[ruled,vlined]{algorithm2e}
\usepackage{algorithmic}

\newcommand{\KL}{\text{KL}}

\newcommand{\T}{\text{\tiny T}}

\title{Robust Entropy-regularized Markov Decision Processes}

\author{%
  Tien Mai \\
  School of Computing and Information Systems\\
  Singapore Management University\\
  \texttt{atmai@smu.edu.sg} \\
  \And
   Patrick Jaillet \\
  Department of Electrical Engineering and Computer Science\\
  Massachusetts Institute of Technology\\
  \texttt{jaillet@mit.edu} \\
}

\begin{document}

\maketitle

\begin{abstract}
 Stochastic and soft optimal policies resulting from entropy-regularized Markov decision processes (ER-MDP) are desirable for exploration and imitation learning applications. 
Motivated by the fact that such policies are sensitive with respect to the state transition probabilities, and the estimation of these probabilities may be inaccurate,  
we study a robust version of the ER-MDP model, where the stochastic optimal policies are required to be robust with respect to the ambiguity in the underlying transition probabilities.
Our work is at the crossroads
of two important schemes in reinforcement learning (RL), namely, robust MDP and entropy-regularized MDP. We show that essential 
properties that hold for the non-robust ER-MDP and robust unregularized MDP models also hold in our settings, making the robust ER-MDP problem tractable. We show how our framework and results can be integrated into different algorithmic
schemes including value or (modified) policy iteration,  which would lead to new robust RL and inverse RL algorithms to handle uncertainties. Analyses on computational complexity and error propagation under conventional uncertainty settings are also provided. 
\end{abstract}

\section{Introduction}
This paper is focused on a robust approach for entropy-regularized Markov Decision Processes (ER-MDP) when the transition probabilities (or dynamics) are themselves uncertain.
By studying the robust ER-MDP framework, we aim at providing a theoretical basis to develop new robust reinforcement learning (RL)/planning algorithms to make decisions under dynamics uncertainty, and more accurate inverse reinforcement learning (IRL) algorithms for solving the problem of reward learning when the experts are conservative with respect to the dynamics uncertainty.  

Robust MDP is an important framework in the machine learning and operations research literature. The framework is motivated by the fact that, in many practical RL/planning problems, the estimation of dynamics might be far from accurate and optimal policies to Markov Decision problems would be very sensitive with respect to these probabilities \citep{mannor2004bias}. The MDP/RL literature has seen a number of solution methods on how to make robust policies in this uncertainty setting \citep{Nilim2004robustness,Lim2013RL_RMDP}. 
On the other hand, ER-MDP is another important scheme in the RL/IRL literature with a different motivation. 
The  framework was first proposed by \cite{Ziebart2008maximum} in the context of IRL, i.e., the problem of recovering an expert's reward function from demonstrations, with the advantage of
removing ambiguity between demonstrations and the expert policy, and casting the reward learning as a maximum likelihood estimation problem.
This framework then became popular in the IRL literature with many successful algorithms  \citep{Levine2011nonlinearIRL,Ho2016generative,Fu2017Robust_IRL}.  To the best of our knowledge, existing IRL frameworks/algorithms all assume that the expert knows the dynamics with certainty. Since it might be not the case and the expert would adapt their decisions with respect to the dynamic uncertainty, e.g., being conservative when making decisions,  ignoring this uncertainty issue in expert's demonstrations would lead to inaccurate reward structures. In addition, the ER-MDP has also been popular in the (deep) reinforcement learning (RL) literature with many state-of-the-art \textit{soft}-RL algorithms, e.g. Soft-Actor-Critic \citep{Haarnoja2018softa,Haarnoja2018softb},
with various motivations such as improving exploration, compositionality and robustness in RL. In fact, since the estimation of the dynamics might not be accurate, robust versions of the ER-MDP framework and  soft-RL algorithms are relevant and worth exploring, noting that such a robust framework has never been formally studied before.    

Given the importance of the ER-MDP framework in the RL/IRL literature  and the issue of facing uncertainties when making policies, a robust approach for ER-MDP  would provide principled answer to the questions of how to recover an accurate reward function from expert's robust/conservative demonstrations, and how to be robust in \textit{soft}-RL algorithms when the dynamics are themselves uncertain.  This motivates us to introduce and study the robust ER-MDP framework,  aiming at proving a complete and rigorous theoretical basic for developing new RL/IRL algorithm that is robust with respect to dynamics uncertainty. More specifically, we explore several aspects such as the duality properties of the robust problem, the complexity of the resulting algorithms and the complexity of the adversary's problem under different uncertainty settings.
We then use these results to design new and efficient robust soft-RL and IRL algorithms. 

Our main contributions in this paper are to show that 
the estimation of the robust optimal policies in robust ER-MDP can be done efficiently \footnote{Here, ``efficiency'' means that ``the worst-case complexity is polynomial time.''}, under conventional uncertainty settings, and the complexity is similar to the case of the robust (unregularized)  MDP and only modestly larger than the case of  non-robust ER-MDP model.   
More specifically,
we consider two  conventional uncertainty settings, i.e., $(s,a)$- and $(s)$-rectangularity \citep{Iyengar2005robustDP,Wiesemann2013robustMDP}  and show that 
some essential  properties such as contraction and Markov optimality hold for the robust ER-MDP model.
These properties are important to design tractable algorithms to solve our robust Markov problems. 
We also point out that the \textit{perfect duality} (or \textit{minimax equality}) that holds for the classical robust MDP model also holds in our setting,
{noting that our robust ER-MDP problem is more challenging to handle and requires new proofs, as many results that hold for the unregularized MDP do not hold for the ER-MDP, e.g. the Markov problem no-longer can be formulated as a linear program.  }
From these basic properties, we  analyze and provide bounds for the computational complexity  and error propagation of value function when the adversary's minimization problems are only solved approximately. We further show how our framework can be used to develop new RL/IRL algorithms.
Moreover, since solving the adversary's minimization problem is a key issue in robust MDP, we extend the results from previous studies \citep{Iyengar2005robustDP,Nilim2004robustness} by considering uncertainty sets based on several KL divergence bounds and show that the resulting minimization problem can be solved efficiently as well. We also provide numerical results to demonstrate  applications of our framework/algorithms in some IRL tasks.

\textbf{Related work:} The machine learning and operations research literatures have seen a number of studies on robust Markov decision processes (MDP) and  reinforcement learning in robust MDP \citep{White1994ROMDP,Iyengar2005robustDP,Nilim2004robustness,Lim2013RL_RMDP,Wiesemann2013robustMDP}. Existing work mostly relies on unregularized MDP, thus makes use of some results that only hold for the unregularized model, e.g., the Markov problem can be formulated as a linear program. On the other hand,  the ER-MDP framework has become popular in both RL and IRL literature. In RL,  \cite{Schulman2015trust} propose a policy iteration scheme, called Trust Region Policy Optimization, in which entropy terms are added to the greedy step to penalize the Kullback-Leibler (KL) divergence between two consecutive policies. The idea  of using entropy regularizers to penalize the divergence between consecutive policies has been also used in Dynamic Policy Programming (DPP) \citep{Azar2012dynamic},  Maximum A Posteriori
Policy Optimization (MPO) \cite{Abdolmaleki2018maximum,Abdolmaleki2018relative}, and robust MPO \cite{Mankowitz2019robust}.    
Some recent RL algorithms have been developed to take  advantage of soft value function and soft policies resulting from  the ER-MDP scheme, for example, Soft-Q learning \citep{Fox2015taming,Schulman2017equivalence,Haarnoja2017RL_ER} and Soft-Actor-Critic \citep{Haarnoja2018softa,Haarnoja2018softb}. 
\cite{Mankowitz2019robust} have added robustness to a \textit{soft}-RL (i.e., MPO) algorithm, but their work is experimentally focused and their theoretical explorations are limited, in the sense that many important aspects such as duality properties, complexity and other uncertainty settings were not investigated.
In IRL, many state-of-the-art algorithms are based on the ER-MDP framework, e.g., Gaussian Process IRL \citep{Levine2011nonlinearIRL} and generative adversarial IRL \citep{Goodfellow2014GAN,Fu2017Robust_IRL,Yu2019multi}.  \citep{viano2020robust} propose a robust IRL algorithm under a dynamic mismatch between the expert and learner, but their settings are different, as they assume that the learner does not knows the expert's dynamics with certainty, while in our context the expert is unsure about the dynamics 
and this information is revealed to the learner. 
Other types of regularizers have been also studied. For example, \cite{Lee2018sparse} propose to use a Tsallis entropy with the motivation of having sparse policies.  \cite{Geist2019theory} propose a general MDP framework regularized by any concave function. We will show that our theoretical results can also be applied to these general settings.

Our paper is structured as follows. Section \ref{sec:problem-description} describes our problem setting and Section \ref{sec:infinite-horizon} presents theoretical properties of the robust ER-MDP model. We discuss related algorithms and frameworks in Section \ref{sec:related-algo}. Section \ref{sec:uncertainty-model} analyses the computational  complexity of the adversary's problems. Section \ref{sec:expr-IRL} provides experiments for robust IRL,  and finally, Section \ref{sec:conclude} concludes.  We provide all the proofs and relevant discussions in the supplementary material. We use $|\cS|$ to denote the cardinality of set $\cS$. Boldface characters represent matrices (or vectors) or a collection of values.



\section{Problem description}\label{sec:problem-description}
 Consider an infinite-horizon Markov decision process (MDP) for an agent with finite states and actions, defined by a tuple $(\cS,\cA,\bQ, \br,\gamma)$, where $\cS$ is a set of states $\cS = \{1,2,\ldots,|\cS|\}$, $\cA$ is a finite set of actions, 
 $\bQ = \{\bq^0,\ldots,\bq^{\infty}\}$ are transition probabilities where
$\bq^t:\cS\times \cA\times\cS \rightarrow [0,1]$ is a transition probability function at time $t$, i.e., $q^t(s_{t+1}|a_t,s_t)$ is the probability of moving to state $s_{t+1}\in\cS$ from $s_t\in \cS$ by performing action  $a_t\in \cA$ at time step $t$, $\br = \{r(a|s),a\in\cA,s\in\cS\}$ is a reward function, and $\gamma\in[0,1]$ is a discount factor. 

Let $\bPi = \{\bpii^0,...,\bpii^\infty\}$ be a policy function where $\pi^t(a_t|s_t)$ is the probability of making action $a_t\in\cA$ at state $s_t\in\cA$ at time $t\in \{0,1,...\}$, the goal of (forward) reinforcement learning under maximum causal entropy principle  is to find an optimal policy $\bPi$ that maximizes the expected entropy-regularized discounted reward \citep{ziebart2010_IRL_Causal,Bloem2014infiniteIRL,Schulman2017equivalence,Haarnoja2017RL_ER}
{\small \begin{align}
&\max_{\substack{\bpii^t \in \Delta^\pi\\ t = 0,1...}}\Bigg\{F_\infty(\bPi,\bQ) = \bbE_{\tau \sim (\bPi,\bQ)}\Big[\sum_{t=0}^{\infty}\gamma^{[t]} r(a_t|s_t)  - \gamma^{[t]} \eta\ln \pi^t(a_t|s_t) \Big]\Bigg\},\label{prob:DET-Max-EP}
\end{align}}
where $\gamma^{[t]}$ refers to ``\textit{$\gamma$ to the power of $t$}'' (we use $[.]$ to distinguish it from a superscript $t$),  $\tau = \{(s_0,a_0),\ldots, (s_\infty,a_\infty)\}$ is a strategy in the infinite-horizon case,
$\Delta^\pi$ is the set of policies 
$\Delta^{\pi} = \{\pi(a|s)\in[0,1]|\ \sum_{a\in\cA} \pi(a|s) = 1,\ \forall s\in\cS\}$, and $\eta\geq 0$ is a regularization coefficient. The term $\cH(\bpii) = -\bbE_{\bPi,\bQ}[\sum_{t=0}^\infty \gamma^{[t]}\eta \ln\pi(a_t|s_t)]$ is referred to as a $\gamma$-discounted causal entropy, distinguishing the entropy regularized with the standard MDP one. This term makes the expected discount rewards no-longer linear in $\bpii^t$, for $1=0,\ldots,\infty.$
 	

In our problem, we assume that the dynamics (i.e., transition probabilities) are  uncertain and  the robust model aims at finding a robust policy that maximizes the worst-case  expected entropy-regularized reward function. The robust problem can be formulated as 
\begin{align}
\max_{\substack{\bpii^t\in \Delta^\pi \\t= 0,1... }}\min_{\substack{\bq^t\in\cQ\\t = 0,1...}} \Big\{F_{\infty}(\bPi,\bQ) \Big\},\label{prob:robust-infinite}
\end{align}
where $\cQ$ is an uncertainty set for the dynamics, defined as
$
  \cQ \subset \Delta^{q} = \{\bq | \ \sum_{s'\in\cS}q(s'|a,s) = 1,\forall (s',s,a) \},
$
and $\bq^t$ is a vector of transition probabilities chosen by the adversary at time step $t = 0,1,...\infty$. 
Here,  the uncertainty set $\cQ$  are assumed to be  (state,action)-wise  or (state)-wise decomposable, i.e..
the uncertainty set $\cQ$ has the form
$\cQ = \otimes_{(s,a)}\cQ_{sa}$ or $\cQ = \otimes_{(s)}\cQ_{s}$, where $\cQ_{sa}$ and $\cQ_s$ are uncertainty sets for the transition probabilities  $\bq_{sa} = \{q(s'|s,a),\forall s'\}$ and $\bq_s = \{q(s'|s,a)|\forall s,a\}$, respectively, for all $a\in\cA,s\in \cS$. We call these assumptions as $(s,a)$- and $(s)$-rectangularity. 
These assumptions have been widely used to derive tractable solutions for  robust MDP problems  \citep{Nilim2004robustness,Iyengar2005robustDP,Wiesemann2013robustMDP}.


\section{Theoretical Properties and Algorithms}\label{sec:infinite-horizon}
We present essential theoretical results for the ER-MDP model. These results are critical for the tractability of the robust  problems. We then also discuss how to compute robust optimal policies and provide complexity analyses.

\subsection{Theoretical Properties}
We will show that some basic results holding for the robust unregularized MDP and non-robust ER-MDP models are 
are also valid for our robust one, making the computation of robust optimal policies tractable. To facilitate our exposition, let us first consider the mapping
 $\cT[V]:\bbR^{|\cS|} \rightarrow \bbR^{|\cS|}$  
\begin{align}
\cT[V](s) = \max_{\bpii \in \Delta^\pi}&\min_{\bq\in\cQ}\Big\{ \psi_s(\bpii,\bq,V) \Big\},\ \forall s\in \cS \label{eq:define-Vs}
\end{align}
where 
$
\psi_s(\bpii,\bq,V) = \bbE_{\bpii_s} [ r(a|s) -\eta\ln \pi(a|s)+ \gamma \bbE_{s'\sim \bq_{sa}} [V(s')] ] 
$. We also define  $\cT^{\bpii}[V] = \min_{\bq\in\cQ}\{ \psi_s(\bpii,\bq,V) \}$ for a fixed policy $\bpii$. On the other hand, let  $V^*,V^{\bpii}:\cS\rightarrow \bbR$ be the expected worst-case accumulated rewards under uncertain transition probabilities (value functions)
\begin{align}
V^{\bpii}(s)& = \min_{\substack{\bq^t\in\cQ\\t = 0,1...}}\Bigg\{ \bbE_{\tau \sim (\bPi,\bQ)} \Bigg[\sum_{t=0}^{\infty}\gamma^{[t]} \Bigg(r(a_t|s_t)- \eta\ln \pi^t(a_t|s_t)\Bigg)\Bigg|
\ s_0 = s\Bigg]\Bigg\}\label{eq:define-V-pi}
\end{align}
and  $V^{*}(s) = \max_{\substack{\bpii^t\in \Delta^\pi,\; t= 0,1... }} V^{\bpii}(s)$, $\forall s\in\cS$.
The following theorem focuses on the $(s,a)$-rectangularity case and  shows some main properties of the robust problem. For notational brevity, let us first denote  $h(a,s|V) = r(a|s) +  \gamma \min_{\bq_{sa}\in\cQ_{sa}}\left\{\bbE[ V(s')]\right\}$ for any $a\in\cA,s\in\cS$.
\begin{theorem}[$(s,a)$-rectangularity]
\label{th:main-sa-rec}
Assume that the uncertainty set $\cQ$ is $(s,a)$-rectangular, $\cT[V]$ and $\cT^{\bpii}[V]$ are contraction mappings of parameters $\gamma$ and $V^*$, and $V^{\bpii}$ are unique solutions to the contraction systems $\cT[V] = V$ and $\cT^{\bpii}[V] = V$, and the mapping $\cT[V]$ can be updated as
$
    \cT[V] = {\eta} \ln\Big(\sum_{a\in\cA} \exp\Big(h(a,s|V)/\eta \Big)\Big)$, 
and the policy $\bpii^*$ defined as $
\pi^*(a|s) = \left(\exp\Big(h(a,s|V^*)/\eta\Big)\right)\Big/\left(\sum_{a'} \exp\Big(h(a',s|V^*)/\eta\Big)\right), \ \forall s\in\cS, a\in\cA
$  is optimal to  \eqref{prob:robust-infinite}. Moreover, the perfect duality holds for both the mapping $\cT[V]$ and $\cT^{\bpii}[V]$  and robust expected reward function, i.e., $\max_{\bpii \in \Delta^\pi}\;\min_{\bq\in\cQ}\Big\{ \psi_s(\bpii,\bq,V) \Big\}  =\min_{\bq\in\cQ} \;  \max_{\bpii \in \Delta^\pi}\Big\{ \psi_s(\bpii,\bq,V) \Big\}$ and $\max_{\bpii^0,\ldots}\min_{\bq^0,\ldots} F_\infty(\bPi,\bQ) = \min_{\bq^0,\ldots} \max_{\bpii^0,\ldots} F_\infty(\bPi,\bQ).$

\end{theorem}
The detailed proof is provided  in the supplementary (Section \ref{sec:proof:th31}). 
The proof for the contraction property of $\cT$ and $\cT^{\pi}$ 
shares the same spirit as in the standard robust MDP model \citep{Iyengar2005robustDP}. The main difference here is the inclusion of the nonlinear entropy term in the Bellman update. The formulation for the robust optimal policy has a similar form as those from the non-robust ER-MDP model, except that instead of performing the Bellman update with known transition probabilities, we need to compute a minimization value $\min_{\bq_{sa} \in\cQ_{sa}} \big\{\bbE_{\bq_{sa}}[V(s')]\}$. The proof can done  using   the fact that the minimization problem $\min_{\bq_{sa}}$ can be put inside the expectation in such a way that it does not depend on the policy $\bpii$, i.e.,
\begin{align}
&\min_{\bq\in\cQ}\Big\{ \psi_s(\bpii,\bq,V) \Big\} = \bbE_{\bpii} \Big[ r(a|s) -\eta\ln \pi(a|s)+ \gamma \min_{\bq_{sa}}\bbE_{s'\sim \bq_{sa}} [V(s')] \Big], \label{eq:sa-equality} 
\end{align}
noting that it is only valid if the uncertainty set $\cQ$ are $(s,a)$-rectangular. \cite{Mankowitz2019robust} give the same formulations, but they do not explicitly show that this solution is optimal to the infinite-horizon problem \eqref{prob:robust-infinite} and corresponds to a saddle point of the \textit{max-min} problem.  
The \textit{perfect duality} property is interesting in the sense that we can swap the \textit{max-min} order even-though the objective functions $\psi_s(\bpii,\bq,V)$ and $F_{\infty}(\cdot)$ are not linear in $\bpii$ and the uncertainty  set $\cQ$ is not necessarily compact and/or convex. We note that
in \cite{Nilim2004robustness}, the perfect duality is proved for standard robust MDP  using linear programming, which is not applicable in our context. 

We now discuss how our results can be extended to the $(s)$-rectangularity case, which allows the transition probabilities $\bq_{sa}$ to be dependent over actions $a\in \cA$. The main difference and challenge lie in the fact that adversary's minimization problem involve the policy variable $\bpii$, making  \eqref{eq:sa-equality} no-longer valid. The following theorem shows how a robust optimal policy in this case can be efficiently computed. First, let $z(a,s|V,\bq) =  r(a|s) +  \gamma \sum_{s'\in\cS} q(s'|a,s) V(s')$ for notational simplicity.
\begin{theorem}
[$(s)$-rectangularity]
\label{theor:optimal-policies-s-rect}
Assume that the uncertainty set $\cQ$ is $(s)$-rectangular   and $\cQ_s$ are compact and  convex,
for all $s\in\cS$, a robust optimal policy $\bpii^*_s$ to \eqref{prob:robust-infinite} can be computed as 
$
\pi^*(a|s) = \left({\exp\big(z(a,s|V^*,\bq^*)/\eta\big)}\right)\big/\left({\sum_{a'} \exp\big(z(a',s|V^*,\bq^*)/\eta \big)}\right), \ \forall s\in\cS, a\in\cA,
$
where $V^*$ is the unique fixed point solution to the contraction mapping ${\cT}[V] = V$, where
${\cT}[V](s) = {\eta} \bigg\{\ln\bigg(\sum_{a\in \cA}\exp\big(z(a,s|V,\bq^*)/\eta \big)\bigg)\bigg\}$, where
and $\bq^*_s$ is an optimal solution to the convex optimization problem
\begin{equation}
\label{eq:mapping-inner-s-rect}
\min_{\bq_s\in\cQ_s} \left\{\sum_{a\in \cA}\exp\Big(z(a,s|V^*,\bq)/\eta \Big)\right\},\;\forall s\in\cS.    
\end{equation}
Moreover, the perfect duality holds.
\end{theorem}
The proof can be found in the supplementary material.
In this setting,  we assume that the uncertainty set if convex and compact in order to use the {Von Neumann’s minimax theorem} to swap the \textit{max-min} order. 
Note that the perfect duality for the robust unregularized MDP model under  $(s)$-rectangular sets has been proved in \cite{Wiesemann2013robustMDP}  using the result that 
the value function can be expressed as $V^{\bpii} = \sum_{t=0}^{\infty}[\lambda\widehat{\bP}^{\bpii}]^{[t]} \widehat{\br}^{\bpii}$, where $\widehat{\bP}^{\bpii}$ is of size $|\cS|\times|\cS|$ with entries $\widehat{\bP}^{\bpii}_{ss'}=\sum_{a}\pi(a|s)q(s'|a,s)$ and $\widehat{\br}^{\bpii}_s = \sum_{a}\pi(a|s)r(a|s)$  \citep{Puterman2014markov}. This result does not apply in our context due to the inclusion of the (nonlinear) entropy terms.
Our idea to derive the formulation for the optimal policy  is that if we swap the \textit{min-max} order, then the inner maximization problem $\max_{\bpii}  \big\{ \psi_s(\pi,\bq,V) \big\}$ will yield a closed-form solution and the corresponding \textit{min-max} problem can be converted into a  \textit{min} problem with a (strictly) convex objective function, which is way easier to solve than the \textit{max-min} counterpart. For this reason, we assume that the uncertainty set $\cQ$ is convex, which is a typical assumption  in the robust optimization literature, and show that a solution to the \textit{min-max} problem is also  optimal to the \textit{max-min} one (i.e., saddle point). This greatly simplifies the computation.


 
In the $(s)$-rectangularity case, the Bellman update requires to solve an exponential convex optimization problem, instead of a linear one.  Under the same uncertainty settings as in the $(s,a)$-rectangularity case (e.g. uncertainty sets based on KL divergence), it seems not possible to efficiently solve the inner problem by bisection. However, it is still possible to solve these problems in polynomial time. We discuss this in detail in Section \ref{sec:uncertainty-model}.

\subsection{Approximate Robust Value Iteration and Complexity Analysis}
\label{subsec:VI}
In this section we discuss the computation of optimal policies under our robust settings.
We focus on value iteration and will talk about other algorithms, e.g., IRL, policy iteration, in the next section.   
We first consider the $(s,a)$-rectangularity case and note that the analysis can be further extended to the $(s)$-rectangularity case.

The contraction mapping implies that the value iteration method converges to a unique fixed point when the number of iterations tends to infinity. Let us define $\cT^n[V] = \cT[\cT^{n-1}[V]] $ for $n=1,2,...$ and $\cT^0[V] = V$, for any $V\in\bbR^{|\cS|}$, and let $V^*$ is the unique fixed point to the mapping $\cT[V] = V$
From the contraction property, it is well-known that, to obtain an $\epsilon$-approximation of the fixed point solution, one would need $(\ln\epsilon^{-1} - \ln ||V^*||_\infty)/\ln\gamma \in \cO(\ln\epsilon^{-1} )$  iterations.

The mapping $\cT[V]$ involves a minimization problem $\min_{\bq_{sa}}\bbE[V(s')]$, which
can be  solved approximately by bisection   \citep{Iyengar2005robustDP,Nilim2004robustness}. In this section we study the complexity to compute  $\epsilon$-approximations of the value function and optimal policy under \textit{soft} Bellman updates. First, to facilitate our exposition, given $\xi>0$, assume that there is an algorithm  of complexity $C(\xi)$ that allows to compute a solution $\overline{\bq}_{sa}$ such that, $\forall a\in\cA,s\in\cS$, 
$
\min_{\bq_{sa}}\bbE_{\bq_{sa}}[V(s')]\geq \bbE_{\overline{\bq}_{sa}}[V(s')] -\xi. \nonumber 
$
We will examine $C(\xi)$ under different uncertainty sets later in Section \ref{sec:uncertainty-model}. Since the adversary's problem can only be solved approximately, $\cT[V]$ and optimal policies are approximated, for any $s\in\cS$, as 
\begin{align}
\widetilde{\cT}[V](s) &= \eta \ln\left(\sum_{a\in \cA}\exp\Bigg(\Big(z(a,s|V,\overline{\bq}) \Big)/\eta\Bigg) \right),\;
\widetilde{\bpii}_s = \frac{\exp\Big(z(a,s|\widetilde{V},\overline{\bq})/\eta\Big)}{\sum_{a'} \exp\Big(z(a',s|\widetilde{V},\overline{\bq})/\eta \Big)}, \label{eq:approx-contraction-mapping}
\end{align}
where $\widetilde{V}$ is an approximate value function. 
 Theorem \ref{th:ER-approx-CM} below examines approximation errors in \eqref{eq:approx-contraction-mapping}. The results differ from standard robust MDP because we have \textit{soft} approximate optimal policies and $\widetilde{\cT}[V]$ involves  \textit{log} and \textit{exp} functions.  
\begin{theorem}
\label{th:ER-approx-CM}
The approximate Bellman update and policy can be bounded as follows
\begin{itemize}
    \item [(i)] $||\widetilde{\cT}^n[V] - \cT^n[V] ||_\infty \leq \xi\gamma (1-\gamma^{[n]})/(1-\gamma)$
    \item [(ii)] For any $\epsilon>0$, if $\xi \leq  \epsilon(1-\gamma)^2/(4\gamma)$ and $||\widetilde{\cT}^{n+1}[V]-\widetilde{\cT}^n[V]||_\infty \leq 3\epsilon(1-\gamma)/4$, then $||\widetilde{\cT}^n[V] -V^*||_\infty \leq \epsilon$
    \item[(iii)] If we compute a soft policy $\widetilde{\bpii}$ by an approximate  value function $\widetilde{V}$ such that $||\widetilde{V}-V^*||_\infty \leq \epsilon$, for an $\epsilon>0$, then $||\widetilde{\bpii}-\bpii^*||_\infty \leq (e^{2(\epsilon+\xi)/\eta}-1)$.
\end{itemize}
\end{theorem}
Theorem \ref{th:ER-approx-CM}-(i)  is useful to analyze the error propagation of value iteration after a certain number of iterations. This bound also holds for policy evaluation. Theorem \ref{th:ER-approx-CM}-(ii) answers the questions of when we should stop the value iteration algorithm to achieve a certain level of accuracy, and Theorem \ref{th:ER-approx-CM}-(iii) shows an approximation error of  the approximate  optimal policy. We also see that one needs  $\cO(\ln\epsilon^{-1})$ iterations to get an approximation in {(ii)}.
We now analyze the computational complexity to get $\epsilon$-approximations of the value function $V^*$ and the optimal policy $\bpii^*$.
According to Theorem \ref{th:ER-approx-CM}-(ii) and analyses from Section \ref{sec:uncertainty-model},  to get an $\epsilon$-approximation of the value function, it would require a worst-case complexity of $\cO(|\cS||\cA|\max_s \{N_s\} (\ln \epsilon^{-1})^2)$ for uncertainty sets based on a single KL divergence bound, and  $\cO(|\cS||\cA|(\max_s \{N_s\})^{7/2} (\ln \epsilon^{-1})^2)$ for the case of several KL divergence bounds, where $N_s$ is the number of states that can be reached from state $s\in\cS$, which is typically much smaller than $|\cS|$. 
On the other hand, to get an $\epsilon$-approximation of  $\bpii^*$, using Theorem \ref{th:ER-approx-CM}-(ii)-(iii),  the computation would need  complexities of $\cO(|\cS||\cA|\max_s \{N_s\}\ln \epsilon^{-1}\ln((\ln (\epsilon+1))^{-1} ) )$ and $\cO(|\cS||\cA|(\max_s \{N_s\})^{7/2}\ln \epsilon^{-1} \ln((\ln (\epsilon+1))^{-1} ) )$ for the cases of single KL bound and several KL bounds, respectively
Note that the robust (unregularized) MDP problem has the same complexity bounds, and  under non-robust ER-MDP the complexity of getting an $\epsilon$- approximation of the value function becomes   $\cO(|\cS||\cA|(\max_s \{N_s\})\ln \epsilon^{-1})$. Thus, adding robustness to ER-MDP yields an extra computational cost  of $\cO(\ln \epsilon^{-1})$ for the case of single KL bound and $\cO((\max_s{N_s})^{5/2}\ln \epsilon^{-1})$ in the case of several KL bounds. 

In the $(s)$-rectangularity case, we perform the Bellman update by solving  \eqref{eq:mapping-inner-s-rect}. 
This is a convex optimization problem and, under some conventional settings, can be solved in polynomial time. Section \ref{sec:uncertainty-model} below  shows that a $\xi$-approximation of the inner minimization problem (with uncertainty sets of several KL bounds) can be achieved with complexity $\cO(N_s^{7/2} \ln\xi^{-1})$. As a result, it would require a complexity of $\cO(|\cS|(\max_s \{|\cA|N_s\})^{7/2} (\ln \epsilon^{-1})^2)$ to have an $\epsilon$-approximation of $V^*$, and a complexity of $\cO(|\cS|(|\cA|\max_s \{N_s\})^{7/2}\ln \epsilon^{-1}\ln((\ln (\epsilon+1))^{-1} ))$ to have an $\epsilon$-approximation of the optimal policy.


\section{Applications}
\label{sec:related-algo}
We discuss relevant frameworks and algorithms that would make use of our results. This shows broad benefits of using the robust ER-MDP formulations in different contexts.

\textbf{{IRL/Imitation Learning under Uncertainty.}}
The  (robust) ER-MDP framework yields soft/randomized optimal policies, making it appealing for imitation learning/IRL
\citep{ziebart2010_IRL_Causal,Ho2016generative,Yu2019multi}, as one can conveniently formulate the reward learning problem as maximum likelihood estimation. 
In general, the robust model will be useful under the assumption that the experts who give demonstrated decisions do not know the transition probabilities with certainty, and make \textit{robust} decisions (looking at worst-case  scenarios), noting that the problem of how to make robust decisions in uncertain environments has been  widely investigated in the literature \citep{White1994ROMDP,Iyengar2005robustDP,Nilim2004robustness,Lim2013RL_RMDP}. 
So, it would be valuable to have an IRL algorithm that is able to learn a reward function from expert's robust decisions. 
Our results allow to cast the reward learning problem as maximum likelihood estimation as in the standard case.  A detailed robust IRL algorithm with a complexity analysis are provided in the supplementary material. We provide numerical experiments in Section \ref{sec:expr-IRL} below to demonstrate the benefits of having such a robust IRL algorithm to recover reward functions from robust policies. 
%
  

\textbf{{RL with KL Divergence Penalties.}}
Solving an ER-MDP problem can provide an optimal policy that is not \textit{too far} from a given  policy.  In a planning context, one might be interested in finding a policy that is not far from a given pre-computed policy $\overline{\bpii}$. This policy $\overline{\bpii}$ may be an outcome of a robust (unregularized) MDP model, but due to some changes to the system (e.g. reward function or uncertainty set $\cQ$), one might need to recompute the robust optimal policy without ending up with a completely new one. To this end, we can penalize the reward function by a KL divergence between the old policy $\overline{\bpii}$ and the new one $\KL(\bpii_s||\overline{\bpii}_s) = \sum_{a}\pi(a|s)\ln \frac{\pi(a|s)}{\overline{\pi}(a|s)}$ , or solve a Markov problem with constraints  $\KL(\bpii_s||\overline{\bpii}_s) \leq \beta$, for a scalar $\beta\geq 0$. The uses of KL divergence penalties or constraints are generally equivalent due to the fact that 
the function $\KL(\bpii_s||\overline{\bpii}_s)$ 
can be moved to the objective function using Lagrange duality. We then can solve following robust ER-MDP problem to obtain a new optimal policy
\begin{equation}
\label{eq:OfflineMDP-KL}
\max_{\substack{\bpii^t\in\Delta^\pi\\t=0,\ldots }}\min_{\substack{\bq^t\in\cQ\\t=0,\ldots }} \Bigg\{ \bbE_{\bPi,\bQ}\Big[\sum_{t=0}^{\infty}\gamma^{[t]} (r(a_t|s_t)  - \eta \KL(\bpii^t_{s_t}||\overline{\bpii}_{s_t})\Big]\Bigg\},
\end{equation}
which yields robust regularized Bellman equation $\cT[V] = \max_{\bpii_s}\min_{\bq_s}\{\bbE[r(a|s)-\eta \ln \frac{\pi(a|s)}{\overline{\pi}(a|s)} + \gamma\bbE_{\bq_{sa}}[V(s')]]\}$. Clearly, if we let $\eta\rightarrow \infty$ then  the optimal solution to  \eqref{eq:OfflineMDP-KL} should approaches $\overline{\bpii}$ and if $\eta\rightarrow 0$ then we retrieve the robust unregularized MDP model.  

ER-MDP has been also used in policy iteration to prevent early convergence to sub-optimal
policies, e.g., \citep{Abdolmaleki2018maximum,Abdolmaleki2018relative,Mankowitz2019robust}. So, it would be useful to use it to compute an optimal policy of a robust (unregularized) MDP while solving a robust regularized Bellman equation at each greedy step of a policy iteration algorithm. To facilitate the idea, let consider the modified policy iteration (MPI) approach \citep{Puterman1978modified}. Under our uncertainty settings, at an iteration $k$ of the robust MPI algorithm, we need to perform
\begin{align}
    (i)\ \bpii^{k+1}_{s} = \text{argmax}_{\bpii_s} \Bigg\{ \min_{\bq_{s}}&\Big\{ \bbE_{\bpii_s}\left[r(a|s)+ \bbE[V^{k}(s')]\right]\Big\}- \eta \KL(\bpii_{s}||\bpii^{k}_s)\Big]\Bigg\}\nonumber \\
    (ii)\ V^{k+1} =(\cT^{\UR,\bpii^{k+1}})^{m}&[V^{k}],\ \nonumber 
\end{align}
where
$
 \cT^{\UR,\bpii^{k+1}} = \min_{\bq_s}\bbE_{\bpii^{k+1}_s,\bq_s}[r(a|s)+\gamma V(s')].
$
Here, $\cT^{\UR,\bpii^k}$ is the robust (unregularized) Bellman update under policy $\bpii^k$ and the entropy term $\eta \KL(\bpii_{s}||\bpii^{k}_s)$ is used to control the distance between $\bpii^k$ and $\bpii^{k+1}$. Note that robust policy iteration algorithms without the KL entropy terms have been studied in some previous work \citep{Kaufman2013robust,Ho2018fast_RMDPs}. Now, from an initial policy $\bpii^0$, the algorithm iteratively find new policy by performing the robust regularized step (i) and robust (unregularized) policy evaluation step (ii). 
Clearly, if $m=1$ we retrieve the robust value iteration considered in Section \ref{subsec:VI}, and with a sufficiently large $m$  we retrieve a policy evaluation step. Since we only solve the inner minimization approximately, it is important to look at the approximation errors of Steps (i) and (ii).    
Theorem \ref{th:ER-approx-CM}-(iii) tells us that if we solve the inner minimization with approximation error $\epsilon>0$, then we can obtain a policy $\tilde{\bpii}^{k+1}$ with approximation error $e^{2\epsilon/\eta}-1$. For Step (ii), the bound in Theorem  \ref{th:ER-approx-CM}-(i) applies to an approximate Bellman update $\widetilde{\cT}^{\UR,\bpii^k}$ where the \textit{min} problem is solved approximately. Hence,  the approximation error for  Step (ii) is $\epsilon\gamma (1-\gamma^m)/(1-\gamma)$. These bounds  allow to analyze the error propagation of  the robust MPI (and thus, its convergence and rate of convergence), analogously to  non-robust MPI algorithms \citep{Scherrer2015approximate,Geist2019theory}, and bound the complexity required for the two steps to get a certain level of accuracy. 

\textbf{{Robust General Regularized MDP}}.
Beside  entropy-regularized models, other types of regularizers have been considered. For example, \cite{Lee2018sparse} propose to use Tsallis entropy with the motivation of
having  sparse optimal policies. \cite{Geist2019theory} study a general version of the ER-MDP model by replacing the entropy terms  by any convex functions of $\bpii_s$. 
It is possible to show that the basic properties mentioned in Theorem  \ref{th:ER-approx-CM} still hold for the robust version of that general model, but the robust Bellman update might have no closed-forms and would be more difficult to perform.
In some cases, one may only do it approximately, thus producing an additional level of approximation to  value/policy iteration. We briefly discuss these in the supplementary (Section \ref{sec:extension-rMDPs}). 

\section{Uncertainty Models}\label{sec:uncertainty-model}
A key issue when solving robust MDP problems is to efficiently solve the adversary's minimization problems.
We discuss this under both $(s,a)$- and $(s)$-rectangularity cases, noting that the objective function in the latter case is exponential.  We focus on uncertainty sets based on KL divergence (relative entropy or likelihood models) due to their appealing statistical properties in modeling uncertainties \citep{Nilim2004robustness,Iyengar2005robustDP}.\ Previous studies show that if the uncertainty set involves only one KL divergence bound, then in the  $(s,a)$-rectangularity  case, the inner minimization can be solved efficiently by bisection. We further extend these results by examining the $(s)$-rectangularity case and uncertainty sets based on several KL divergence bounds, with the motivation of better use the availability of historical data  and migrate the conservativeness of the uncertainty sets. 
We present our main results below and refer the reader to the supplementary (Section \ref{proof:Complexity}) for detailed proofs. 

\textbf{{$(s,a)$-rectangular uncertainty sets}}.
In this setting,  the objective function of the inner minimization problem is linear in $\bq_{sa}$. In a likelihood model  or relative entropy model, the uncertainty sets $\cQ_s$ are determined by KL divergence constraints of the forms $\sum_{s'}\hat{q}(s'|s,a)\ln q(s'|s,a)\geq \beta$ or $\sum_{s'}q(s'|s,a) \ln (q(s'|s,a)/\hat{q}(s'|s,a))\leq \beta$, respectively, where $\hat{q}(s'|s,a)$ is an empirical estimate of the transition probability associated with states $s,s'\in\cS$ and action $a\in\cA$. In \cite{Iyengar2005robustDP} the authors show that if  $\cQ_{sa}$ is defined based on only one KL bound, then the inner problem can be solved by bisection with complexity $C(\xi) = \cO(N_s\ln\xi^{-1})$. If we define uncertainty sets using several KL bounds, the bisection no-longer works. However, using some results from convex programming \citep{nemirovski2004interior}, we can show that if the number of KL bounds is much less than $N_s$ (which is typically the case), then the inner problem can be solved by interior-point  with complexity $\cO(N_s^{7/2}\ln\xi^{-1})$.

\textbf{{$(s)$-rectangular uncertainty sets.}}
In this case, the inner minimization problems involve exponential (convex) objective function: $\min_{\bq_s}\sum_a\exp(r(a|s)+ \bbE_{\bq_s}[V(s')])$, making it not solvable by bisection, even with uncertainty sets of only one KL bound. In this context, the problem still can be solved efficiently by interior-point and  it is possible to show that, if the number of KL bounds is much smaller than $N_s$, then the complexity can be bounded by $\cO((|\cA|N_s)^{7/2} \ln \xi^{-1})$, for which we provide a detailed proof in the supplementary. If the number of KL bounds is significant, then we also provide a detailed bound for the complexity in the supplementary material. 
It is worth noting that in a general regularized MDP model \citep{Geist2019theory}, there might be no closed-form for 
$\max_{\bpii_s}\{\cdot\}$ problems, thus one needs   to solve $\max_{\bpii_s}\min_{\bq_s}\{\cdot\}$  to perform the Bellman update, for which a  saddle-point algorithm would be useful \citep[e.g.][]{Gidel2016frank}, but the complexity is not easy to bound. 

\vspace{-0.3cm}
\section{Experiments with Robust IRL}
\label{sec:expr-IRL}
We provide numerical experiments to demonstrate the application of our robust ER-MDP models and algorithms in IRL
Here we focus on IRL, noting that  extensive experiments for a robust soft-RL were provided in \cite{Mankowitz2019robust}.   
We employ the MaxEnt algorithm \citep{Ziebart2008maximum}, one of the most popular IRL algorithms in the literature. 
We assume that the experts are uncertain about the dynamics  and make robust decisions and our aim is  to recover the experts' reward function from such robust decisions. In this context, the standard MaxEnt algorithm will ignore the uncertainties and tries to learn the reward function using a fixed vector of transition probabilities, and our robust version (named as Robust MaxEnt) will explicitly take the uncertainty issue into consideration.   

To evaluate how each algorithm performs, in analogy to prior IRL work \citep{Levine2011nonlinearIRL}, we  use the \textit{``expected value difference''} score, which measures how a learned policy performs under the true rewards.
We will evaluate  IRL outputs  on both environments on which they were learned and random environments (denoted by ``\textit{transfer}''). For the latter, we bring the learned parameters of the rewards to compute rewards and optimal policies in the new environments.  
We will test our robust IRL algorithm using two simulated environments, i.e., Objective-world and Highway Driving Behavior. Brief descriptions are given below.

The \textbf{Objectworld} is an $N\times N$ grid of states in which objects are randomly placed. Each object is assigned one of $C$ inner and outer colors. At each state, there are five possible actions corresponding to staying at the same place or stepping to four different directions (up, down, left, right). 
For the \textbf{Highway Driving Behavior} environment,  the task is to navigate a vehicle in a highway of three lanes with all vehicles moving with the same speed. The agent's vehicle can switch lanes and drive at up to four times speed of the traffic. Other vehicles (motorcycle or car) are civilian or police, and are placed randomly on the three lanes. The agent can make five different actions of changing lanes (left or right), speeding or slowing down, or staying at the same lane and same speed. 
Demonstrations are paths of length 8 generated by the true rewards, true dynamics and robust behavior. We generate 128 samples for each score measures  and we repeat the training and evaluation 8 times to compute the means and standard errors of the scores. We test the algorithms with two ways of generating expert trajectories, that is, standard unregularized MDP (\textit{deterministic policy}) and ER-MDP (\textit{stochastic policy}). 
We use the code and data used in  \cite{Levine2011nonlinearIRL} and keep the the same settings. The experiments were conducted using a PC running Window 10 with Intel(R) CoreTM i7-7700HQ CPU (2.80Hz) and 16 GB RAM.

We define the uncertainty sets as
$
\cQ_{sa} = \{\bq_{sa}\; |  \KL(\bq_{sa}||\bq^0_{sa}) \leq \epsilon \},\;\forall s\in\cS, a\in\cA
$
where  $\bq^0_{sa}$ are the ``\textit{true}'' transition probabilities, noting that these true values are not known by the experts
and we used $\bq^0_{sa}$ to
compute the \textit{expected value difference} scores. In this context, $\epsilon$ represents an uncertainty level. That is, larger $\epsilon$ will correspond to more uncertain expert behavior. We vary $\epsilon$ from $0$ to $0.1$ to show the performance of the  robust IRL algorithm with different $\epsilon$.     

\begin{figure}[htb] 
\centering
    \includegraphics[width=0.8\linewidth]{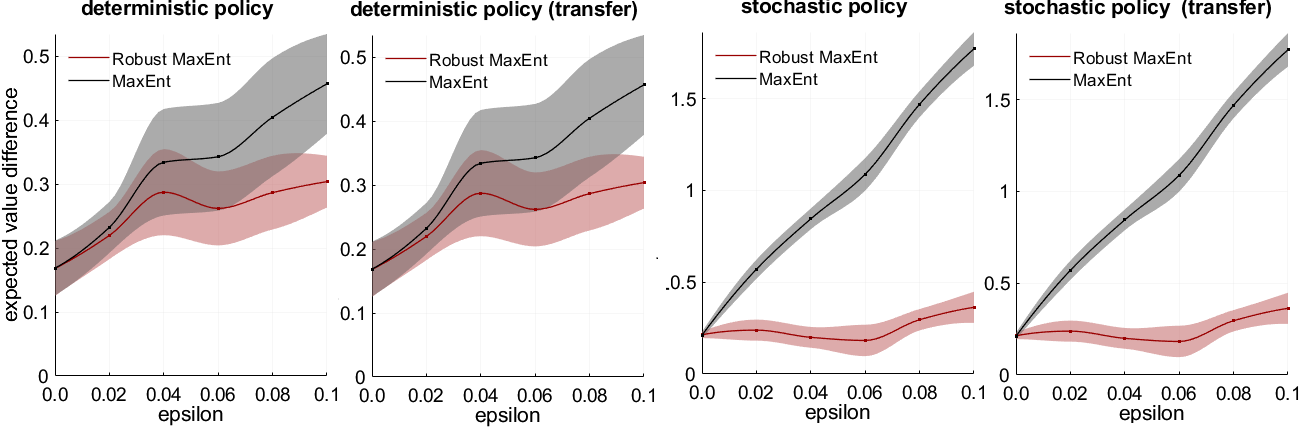} 
    \vspace{-0.3cm}
    \caption{\small Experiments with  objectworld, solid curves show the mean and the shading shows the standard errors.  } 
    \label{fig:obj} 
\end{figure}

\begin{figure}[htb] 
\centering
    \includegraphics[width=0.8\linewidth]{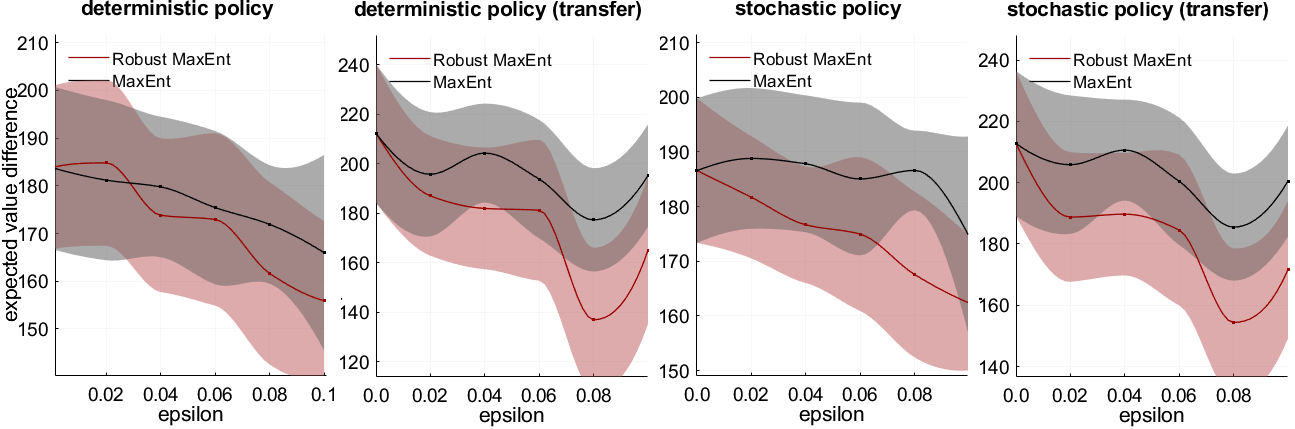} 
    \vspace{-0.3cm}
    \caption{\small Highway driving behavior experiments, solid curves shows the means and shading shows standard errors.  } 
    \label{fig:high} 
\end{figure}

The means and standard errors  of the expected value difference scores for the Objectworld and Highway Driving Behavior environments  are plotted in Fig. \ref{fig:obj} and Fig. \ref{fig:high}, in which the lower the better. It is clear that the Robust MaxEnt constantly outperforms the standard MaxEnt for all the tests, especially for the Objectworld example. The performance gap also increases when $\epsilon$  grows, demonstrating the consequences of ignoring the uncertainty issues in IRL. 
\vspace{-0.3cm}
\section{Conclusion}
\label{sec:conclude}
We study a robust ER-MDP model, aiming at taking the advantages of both robust MDP and ER-MDP schemes to develop new algorithms for robust decision-making and learning. 
We show that several properties that hold in the robust- and ER-MDP models also hold in the robust ER-MDP one. 
From that,
we look at the computation of robust optimal policies and providing computational complexity and error propagation analyses.
We  show how our robust framework can be used to design robust IRL and robust policy iteration algorithms under dynamics uncertainty. We provide numerical experiments to demonstrate the application of our frameworks/algorithms in the context of IRL.

{In this paper we focus on planing settings, i.e., all the information of the MDP is given, except that the dynamics are only known partially. In the context that the environment is unknown and one needs to interact with it to make policies, the algorithms and approximation bounds would need additional work and we keep this for future work. Our robust model might be conservative, in the sense that we assume the dynamics can take any values in the uncertainty set and the uncertainty set needs to satisfy some rectangularity assumptions. Some ways to relax these assumptions are to use distributionally robust approaches  \citep{Xu2010distributionally} and/or k-rectangular robust MDP \citep{Mannor2012lightning}, which would be promising for future work. } 

\bibliographystyle{plainnat}
\bibliography{refs}

\appendix

\clearpage

\begin{center}
\textbf{\huge Supplementary Material}
\end{center}

\section{Proofs of Main Results}
\subsection{Proof of Theorem \ref{th:main-sa-rec}}
\label{sec:proof:th31}

\textbf{Contraction property. }
We prove  the \text{\textit{Contraction}} property.
For any $s\in \cS$, let us consider two cases $\cT[V](s)\geq \cT[V'](s)$ or $\cT[V](s)< \cT[V'](s)$.
If $\cT[V](s) \geq \cT[V'](s)$. For any $\epsilon>0$, let $\bpii^* \in \Delta^\pi$ be  a solution such that
\[
\cT[V](s) \leq \min_{\bq_{s}\in\cQ_{s}}\Bigg\{ \bbE_{\bpii^*} \left[ r(a|s) -\eta\ln \pi^*(a|s) + \gamma \bbE_{s'\sim \bq_{sa}} [V(s')] \right]  \Bigg\} +\epsilon.
\]
Since, $\cT[V](s) \geq \cT[V'](s)$, we have
\begin{align}
    |\cT[V](s)-\cT[V'](s)| &\leq \min_{\bq_{s}\in\cQ_s}\Bigg\{ \bbE_{\bpii^*_s} \left[ r(a|s) -\eta\ln \pi^*(a|s) + \gamma \bbE_{s'\sim q} V(s') \right]  \Bigg\}+\epsilon \nonumber \\
    &\qquad \qquad - \min_{\bq_{s}\in\cQ_s}\Bigg\{ \bbE_{\bpii^*_s} \left[ r(a|s) -\eta\ln \pi^*(a|s) + \gamma \bbE_{s'\sim q} V'(s') \right]  \Bigg\}\nonumber \\
    &= \min_{\bq_{s}\in\cQ_s}\Bigg\{ \gamma \bbE_{\bpii^*_s, s'\sim \bq_s} [V(s')] \Bigg\} - \min_{\bq_{s}\in\cQ_s}\Bigg\{ \gamma \bbE_{\bpii^*_s, s'\sim \bq_s} [V'(s')] \Bigg\} +\epsilon. \label{eq:eq2}
\end{align}
We see that $\min_{\bq_{s}\in\cQ_s}\left\{ \gamma \bbE_{\bpii^*_s, s'\sim \bq_s} V(s') \right\} \leq  \min_{\bq_{s}\in\cQ_s}\left\{ \gamma \bbE_{\bpii^*_s, s'\sim \bq_s} V'(s') \right\}$, so if we denote by $\bq^*_s \in\cQ_s$ a  solution such that
\[
\min_{\bq_{s}\in\cQ_s}\left\{ \gamma \bbE_{\bpii^*_s, s'\sim \bq_s} [V'(s')] \right\} \geq \bbE_{\bpii^*_s, s'\sim \bq^*_s} [V'(s')] -\epsilon, 
\]
then from \eqref{eq:eq2}  we have
\[
|\cT[V](s)-\cT[V'](s)|\leq \gamma \bbE_{\bpii^*_s, s'\sim \bq^*_s} (V(s')-V'(s')) +2\epsilon \leq \gamma ||V-V'||_\infty+2\epsilon, \forall s\in \cS.
\]
Let $\epsilon\rightarrow \infty$ we obtain $||\cT[V] - \cT[V']||_\infty \leq ||V-V'||_\infty$.
The case $\cT[V](s)<\cT[V'](s)$ can be done in a similar way. So $\cT[V]$ is a contraction. The contraction property of   $\cT^{\bpii}[V]$ can be proved in a similar way.

\textbf{Markov optimality. }
We will now prove that if $V^*$ is a unique solution to the contraction mapping $\cT[V] = V$, then $V^*$ will satisfies  $V^{*}(s) = \max_{\bpii^t \in \Delta^\pi, t=0,1,...} V^{\bpii}(s)$, $\forall s\in\cS$. To this end, 
 we first denote $h(a_t,s_t) = r(a_t|s_t)-\eta\ln\pi^t(a_t|s_t)$, $s\in\cS, a\in \cA$. For any policy $\bpii^0,\bpii^1,... \in\Delta^\pi$, from the definition of $\cT[V]$ we have
\begin{align}
    V^*(s) &\geq  \min_{\bq^0_s} \Big\{\bbE_{\bpii^0_s}\left[h(a_0,s_0) + \gamma \bbE_{s_1\sim \bq^0_{s_0a_0}}  [V^*(s_1)]\right]\Big\} \nonumber\\
    &\geq \min_{\bq^0,\bq^1}\left\{ \bbE_{\substack{\bpii^0,\bpii^1\\ \bq^0,\bq^1}}\left[h(a_0,s_0) + \gamma h(a_1,s_1)\right] +
     \gamma^2\bbE_{\bq^1} [V^*(s_2)| s_0=s]\right\}. \nonumber 
\end{align}
This leads to, for any $n\in\mathbb{N}$,
\begin{align}
    V^*(s) &\geq \min_{\bq^0,...,\bq^n}\left\{ \bbE_{\substack{\bpii^0,...,\bpii^n \\\bq^0,...,\bq^n}}\left[ \sum_{t=0}^n \gamma^{[t]} h(a_t,s_t) \Big| s_0=s  \right] + \bbE_{\substack{\bpii^t,\bq^t\\t=0,1,...,n}} \left[\gamma^{[n+1]} V^*(s_{n+1}) \Big|\ s_0=s\right] \right\} \nonumber \\
    & =  \min_{\bq^0,...}\Bigg\{ \bbE_{\substack{\bpii^t,\bq^t\\t=0,1,...}}\left[ \sum_{t=0}^\infty \gamma^{[t]} h(a_t,s_t) \Big| s_0=s  \right] + \bbE_{\substack{\bpii^t,\bq^t\\t=0,1,...,n}} \left[\gamma^{[n+1]} V^*(s_{n+1}) \Big|\ s_0=s\right] \nonumber \\
    & \qquad- \bbE_{\substack{\bpii^{n+1},... \\\bq^{n+1},...}}\left[ \sum_{t=n+1}^\infty \gamma^{[t]} h(a_t,s_t) \right]\Bigg\} \nonumber \\
    & \geq \min_{\bq^0,...}\left\{ \bbE_{\substack{\bpii^0,...\\\bq^0,...}}\left[ \sum_{t=0}^\infty \gamma^{[t]} h(a_t,s_t) \Big| s_0=s  \right]\right\} -  \gamma^{[n+1]} ||V^*||_\infty -  \frac{\gamma^{[n+1]}H}{1-\gamma}, \nonumber
\end{align}
where $||V^*||_\infty = \max_s V^*(s)$ and $$H = \max_{\substack{a_t\in \cA,s_t\in\cS\\ \bq^0,...,\bq^t}} \bbE_{\substack{\bpii^0,...,\bpii^t \\ \bq^0,...,\bq^t}}[r(a_t|s_t) - \eta\ln \pi(a_t|s_t)| s_0=s].$$ We can show that $H<\infty$ because
\begin{align}
H &\leq  \max_{\substack{a_t\in \cA,s_t\in\cS\\ \bq^0,...,\bq^t}} \bbE_{\substack{\bpii^0,...,\bpii^t \\ \bq^0,...,\bq^t}}[r(a_t|s_t) \big| s_0=s] + \bbE_{\substack{\bpii^0,...,\bpii^t \\ \bq^0,...,\bq^t}} [-\eta\ln \pi(a_t|s_t)| s_0=s]\nonumber\\
&\leq R + \max_{\substack{a_t\in \cA,s_t\in\cS\\ \bq^0,...,\bq^t}} -\eta\pi(a_t|s_t)\ln \pi(a_t|s_t)\nonumber \\
&\leq R + \eta/e,\label{eq:eq1}
\end{align}
where $R = \max_{a,s}r(a|s)$ and $e$  the base of the natural logarithm ($e\approx 2.7828$). Inequality \eqref{eq:eq1} is due to the fact that $-x\ln x\leq e^{-1}$ for all $x\in [0,1]$. So we have 
\begin{equation}
\label{eq:ineq1}
V^*(s) \geq \max_{\bpii^0,...}\min_{\bq^0,...} \left\{ \bbE_{\substack{\bpii^0,...\\\bq^0,...}}\left[ \sum_{t=0}^\infty \gamma^{[t]} h(a_t,s_t) \Big| s_0=s  \right]\right\} -  \gamma^{[n+1]} ||V^*||_\infty -  \frac{\gamma^{[n+1]}H}{1-\gamma}, 
\end{equation}
noting that $||V^*||_\infty$ is also finite.
On the other hand, there is a policy $\overline{\bpii}^0,\overline{\bpii}^1,\ldots $  such that, for any $\epsilon>0$ and any $s\in\cS$
\[
V^*(s) \leq \min_{\bq^0} \left\{\bbE_{\overline{\bpii}^0} \left[ h(a_0,s) + \gamma \bbE_{s'\sim \bq^0_{sa}} [V(s')]   \right] \right\} + \epsilon.
\]
We can expand the Bellman equation to obtain
\begin{align}
V^*(s)  &\leq \min_{\bq^0} \left\{\bbE_{\overline{\bpii}^0} \left[ h(a_0,s) + \gamma \bbE_{s_1\sim \bq^0_{s_0a_0}} [V(s_1)]   \right] \right\} + \epsilon  \nonumber \\
&\leq \min_{\bq^0,\bq^1} \left\{\bbE_{\substack{\bq^0,\bq^1 \\\overline{\bpii}^0, \overline{\bpii}^1}} \left[ h(a_0,s_0) + \gamma h(a_1,s_1) \big| s_0=s\right] + \gamma^2 \bbE_{\substack{\bq^0,\bq^1 \\\overline{\bpii}^0, \overline{\bpii}^1}}[V(s_2)\big| s_0=s]    \right\} + (1+\gamma)\epsilon.\nonumber
\end{align}
Thus, by continuing expanding the inequality, we have, for any $n\in\mathbb{N}$,
\begin{align}
V^*(s) &\leq \min_{\bq^0,...,\bq^n} \left\{\bbE_{\substack{\bq^0,...,\bq^n\\\overline{\bpii}^0,...,\overline{\bpii}^n}} \left[ \sum_{t=0}^n \gamma^{[t]} h(a_t,s_t) \big| s_0=s\right] + \gamma^{[n+1]} \bbE_{\substack{\bq^0,...,\bq^n\\\overline{\bpii}^0,...,\overline{\bpii}^n}}[V^*(s_{n+1})\big| s_0=s]    \right\} \nonumber\\& \qquad + \frac{1-\gamma^{[n+1]}}{1-\gamma}\epsilon \nonumber \\
&\leq \min_{\bq^0,...} \Bigg\{\bbE_{\substack{\bq^0,...\\\overline{\bpii}^0,...}} \left[ \sum_{t=0}^\infty \gamma^{[t]} h(a_t,s_t)\big| s_0=s \right] + \gamma^{[n+1]} ||V^*||_\infty -\nonumber \\&\qquad\qquad\qquad \bbE_{\substack{\bq^0,...\\\overline{\bpii}^0,...}} \left[ \sum_{t=n+1}^\infty \gamma^{[t]} h(a_t,s_t)\big| s_0=s \right]   \Bigg\}  + \frac{1-\gamma^{[n+1]}}{1-\gamma}\epsilon \nonumber \\
&\leq \min_{\bq^0,...} \left\{\bbE_{\substack{\bq^0,...\\\overline{\bpii}^0,...}} \left[ \sum_{t=0}^\infty \gamma^{[t]} h(a_t,s_t)\big| s_0=s \right] \right\}+ \gamma^{[n+1]} ||V^*||_\infty + \frac{\gamma^{[n+1]}H}{1-\gamma}   + \frac{1-\gamma^{[n+1]}}{1-\gamma}\epsilon. \nonumber
\end{align}
So we have
\begin{align}
V^*(s) 
&\leq \max_{\bpii^0,...} \min_{\bq^0,...} \left\{\bbE_{\substack{\bpii^0,...\\\bq^0,...}} \left[ \sum_{t=0}^\infty \gamma^{[t]} h(a_t,s_t) \big| s_0=s\right] \right\}+ \gamma^{[n+1]} ||V^*||_\infty + \frac{\gamma^{[n+1]}H}{1-\gamma}   + \frac{1-\gamma^{[n+1]}}{1-\gamma}\epsilon  \label{eq:ineq2}
\end{align}
From \eqref{eq:ineq1} and \eqref{eq:ineq2}, we can take $n\rightarrow \infty$ and $\epsilon$ arbitrarily small, we have 
\[
V^*(s) = \max_{\bpii^0,...} \min_{\bq^0,...} \left\{\bbE_{\substack{\bpii^0,...\\\bq^0,...}} \left[ \sum_{t=0}^\infty \gamma^{[t]} h(a_t,s_t) \Big| s_0=s\right] \right\},
\]
as desired. From the proof, we can also validate see that if $V^{\bpii}$ is a solution to the system $\cT^{\bpii}[V] = V$, then it also satisfies Eq. \ref{eq:define-V-pi}. 

\textbf{Perfect duality.} 
We move to the \text{\textit{Perfect Duality}} property. We first prove that this property holds for the Bellman update. The weak duality implies that 
\[
   \min_{\bq_s\in\cQ_s} \max_{\bpii_s\in \Delta^\pi_s} \psi_s(\bpii,\bq,V)  \geq  \max_{\bpii_s\in \Delta^\pi_s} \min_{\bq_s\in\cQ_s} \psi_s(\bpii,\bq,V)
\]
To prove the opposite inequality, for any $\epsilon>0$,  let  $\overline{\bq}_s \in\cQ_s$ be a transition solution such that
\[
\min_{\bq_{sa} \in\cQ_{sa}}\left\{\bbE_{s'\sim \bq_{sa}}\left[{V}(s')]\right]\right\} \geq \bbE_{s'\sim \overline{\bq}_{sa}}\left[{V}(s')\right] - \epsilon.
\]
We have the chain
\begin{align}
&\min_{\bq_s\in\cQ_s} \max_{\bpii_s\in \Delta^\pi_s} \psi_s(\bpii,\bq,V) \leq  \max_{\bpii_s\in \Delta^\pi_s} \psi_s(\bpii,\bar{\bq},V) \nonumber \\
&= \max_{\bpii_s\in \Delta^\pi_s} \Bigg\{\bbE_{\bpii_s}\left[ r(a|s) -\eta\ln\pi(a|s) +  \bbE_{s'\sim \overline{\bq}_{sa}}[{V}(s')]\right]\Bigg\}\nonumber\\
&\leq \max_{\bpii_s\in \Delta^\pi_s} \Bigg\{\bbE_{\bpii_s}\left[ r(a|s) -\eta\ln\pi(a|s) +  \min_{\bq_{sa}}\bbE_{s'\sim \bq_{sa}}[\overline{V}(s')]\right]\Bigg\} +  \epsilon\nonumber\\
& =  \max_{\bpii_s\in \Delta^\pi_s}  \min_{\bq_{s}\in\cQ_s} \Bigg\{\bbE_{\bpii_s}\left[ r(a|s) -\eta\ln\pi(a|s) + \bbE_{s'\sim \bq_{sa}}[{V}(s')]\right]\Bigg\} +  \epsilon.\nonumber
\end{align}
Let $\epsilon\rightarrow 0$ we obtain the opposite-side of the inequality, which implies the perfect duality for the Bellman update. To prove  the perfect duality for $\max_{\bpii}\min_{\bq} F_{\infty}(\bpii,\bq)$  we define the dual of the mapping $\cT[V]$ as $\widehat{\cT}[V] = \max_{\bpii}\min_{\bq} \psi_s(\bpii,\bq,V)$. Since $\cT[V] = \widehat{\cT}[V]$ for any $V\in\bbR^{|\cS|}$, they yield the same fixed point solution $V^*$. Now, similarly to the proof of the Markovian Optimality property, we can also show that $V^*$ will satisfy
\[
V^*(s) = \min_{\bq^0,...}\max_{\bpii^0,...}  \left\{\bbE_{\substack{\bpii^0,...\\\bq^0,...}} \left[ \sum_{t=0}^\infty \gamma^{[t]} h(a_t,s_t) \Big| s_0=s\right] \right\}.
\]
Combining with the \textit{Markov Optimality } shown in point (iv), we obtain the minimax equality
\[
\max_{\bpii^0,...}\min_{\bq^0,...} F_{\infty}(\bPi,\bQ) =  \min_{\bq^0,...}\max_{\bpii^0,...}F_{\infty}(\bPi,\bQ), 
\]
which completes the proof. 

To compute an optimal policy to the Markov problem, according to the Markov Optimality property of Theorem \ref{th:main-sa-rec}, we just need to find a solution by solving the Bellman equation $\cT[V] = V$. 
Given any state $s\in \cS$, we have
\begin{align}
    \max_{\bpii_s}\min_{\bq_{s}} &\Bigg\{\bbE_{\bpii}\left[ r(a|s) -\eta\ln\pi(a|s) + \bbE_{\bq_{sa}} [V(s')]\right]\Bigg\}\nonumber\\
    &=\max_{\bpii_s} \Bigg\{\bbE_{\bpii_s}\left[ r(a|s) -\eta\ln\pi(a|s) + \gamma \min_{\bq_{sa}} \bbE_{\bq_{sa}}[V(s')]\right]\Bigg\}.\nonumber
\end{align}
So we can write 
\begin{align}
    V(s) &= \max_{\bpii_s} \left\{\sum_{a\in\cA} {\pi}(a|s)\Big( r(|s)-\eta\ln {\pi}(a|s) + \delta(s,a)\Big) \right\}, \nonumber
\end{align}
where $\delta(s,a) =  \min_{\bq_{sa}} \sum_{s'\in\cS} q(s'|s,a) V(s')$ for notational brevity. Let consider the  maximization problem  
 \begin{align}
  J(s) = 	\underset{\bpii_s}{\text{max}}\qquad &  \sum_{a\in\cA} {\pi}(a|s)\Big( r(a|s)-\eta\ln {\pi}(a|s) + \delta(s,a)\Big) & \label{prob:max-entropy-state-1} \\
	 \text{subject to} \qquad & \sum_{a\in\cA}\pi(a|s) = 1 &  \nonumber\\
	  &  \pi(a|s)\geq 0,\ \forall a\in\cA.  &  \nonumber
\end{align}
We consider the Lagrange function
\[
L(\bpii_s,\beta) = \sum_{a\in\cA} {\pi}(a|s)\Big( r(a|s)-\eta\ln {\pi}(a|s) + \delta(s,a)\Big) - \beta\left(\sum_{a}\pi(a|s)-1\right).
\]
We see that if $\bpii^*_{s}$ is optimal to \eqref{prob:max-entropy-state-1}, then  $(\partial L(\pi(\cdot|s),\beta))/\partial \pi(a|s)=0$  at $\pi^*(a|s)$ for all $a\in\cA$, leading to the following equations
\begin{equation}
\label{eq:ER-optimalpolicy-finite}
\begin{cases}
\eta\ln \pi^*(a|s) = r(a|s) + \delta(s,a) - (\eta+\beta) \\
\sum_{a}\pi^*(a|s) = 1.
\end{cases}
\end{equation}
Hence, we have
\[
\begin{cases}
\pi^*(a|s) = (\exp(r(a|s)/\eta + \delta(s,a)/\eta))/{\exp(1+\beta/\eta)} \\
\exp(1+\beta/\eta) = \sum_{a}\exp(r(a|s)/\eta + \delta(s,a)/\eta)\\
J(s) = (1+\beta/\eta) = \ln\left(\sum_{a}\exp(r(a|s)/\eta + \delta(s,a)/\eta)\right). 
\end{cases}
\]
This leads to a closed form to compute the objective of the  maximization problem. The value of $V^t$ becomes
\[
\begin{aligned}
\cT[V](s) &= \eta\ln\left(\sum_{a}\exp\left(r(a|s)/\eta +  \frac{1}{\eta} \min_{\bq_{sa}}\sum_{s'\in\cS} q(s'|s,a) V(s')\right)\right)\\  &= \eta\ln\left(\sum_{a}e^{V(a|s)}\right),
\end{aligned}
\]
where
\[
V(a|s) = r(a|s)/\eta+\frac{1}{\eta} \min_{\bq_{sa}} \sum_{s'\in\cS} q(s'|s,a) V(s').
\]
The optimal policies $\pi^*(a|s)$ then becomes ${ \exp\left(V(a|s)\right)}/{\exp\left(\cT[V](s)/\eta\right)}$, according to \eqref{eq:ER-optimalpolicy-finite}. We obtain the desired equations for both  $\cT[V]$ and an 
optimal policy $\bpii^*$. 

\subsection{Proof of Theorem \ref{theor:optimal-policies-s-rect}}  
\textbf{Contraction property.}  We also consider two cases. If $\cT[V](s) \geq \cT[V'](s)$. For any $\epsilon>0$, let $\bpii^* \in \Delta^\pi$ be  a solution such that
\[
\cT[V](s) \leq \min_{\bq_{s}\in\cQ_{s}}\Bigg\{\psi_s(\bpii^*,\bq,V) \Bigg\} +\epsilon.
\]
We have
\begin{align}
   0\leq  \cT[V](s)-\cT[V'](s) &\leq \min_{\bq_{s}\in\cQ_s}\Big\{\psi_s(\bpii^*,\bq,V) \Big\}+\epsilon 
   - \min_{\bq_{s}\in\cQ_s}\Big\{\psi_s(\bpii^*,\bq,V')\Big\}. \label{eq:s-eq2}
\end{align}
So if we denote by $\bq^*_s \in\cQ_s$ a  solution such that
\[
\min_{\bq_{s}\in\cQ_s}\Big\{\psi_s(\bpii^*,\bq,V')\Big\} \geq \psi_s(\bpii^*,\bq^*,V')-\epsilon, 
\]
then from \eqref{eq:s-eq2}  we have
\[
|\cT[V](s)-\cT[V'](s)|\leq\psi_s(\bpii^*,\bq^*,V) - \psi_s(\bpii^*,\bq^*,V') +2\epsilon \leq \gamma ||V-V'||_\infty+2\epsilon, \forall s\in \cS.
\]
Let $\epsilon\rightarrow \infty$ we obtain $||\cT[V] - \cT[V']||_\infty \leq ||V-V'||_\infty$.  The case $\cT[V](s)< \cT[V'](s)$ is similarly proved.

Given the contraction property, the Markov Optimality (iv) can be validated similarly as in  the $(s,a)$-rectangularity case.

\textbf{Perfect duality. }
For the perfect duality property, noting that the variables  $\bpii_s$ in the adversary's problem  $\min_{\bq_s}\{\cdot\}$ cannot be  eliminated as in the $(s,a)$-rectangularity case. However, with the assumption that the uncertainty set is \textit{convex} and \textit{compact}, we can make use of the  \textit{von Neumann's minimax theorem        } \citep{Brandt2016Minimax} and the fact that function $\psi_s(\bpii,\bq,V)$ is linear in $\bpii_s$ and convex in $\bq_s$, 
to see that
\[
\max_{\bpii_s}\min_{\bq_s} \Big\{\psi_s(\bpii,\bq,V)\Big\} = \min_{\bq_s} \max_{\bpii_s} \Big\{\psi_s(\bpii,\bq,V)\Big\},
\]
Thus the Perfect Duality holds for Bellman equation $\cT[V]$, $\cT[V] = \widehat{\cT}[V]$ for all $V\in\bbR^{|\cS|}$. Using the contraction property, let $V^*$ be a unique fixed point solution to the systems  $\cT[V] = V$ and $\widehat{\cT}[V] = V$, then similarly to proof of Theorem \ref{th:main-sa-rec}-(v), we can show that the Perfect Duality also holds for  $\max_{\bpii^0,\ldots}\min_{\bq^0,\ldots}F_{\infty}(\bPi,\bQ)$. 

\textbf{Optimal policy. } 
The key issue/challenge when proving the formulation for the optimal policy is to show that if $(\bpii^*_s,\bq^*_s)$ is a solution to the  Bellman update $\max_{\bpii_s}\min_{\bq_s} \Big\{\psi_s(\bpii,\bq,V)\Big\}$, then is is also solution to the \textit{min-max} counterpart. It is not always the case even if the perfect duality (or minimax equality) holds. In this proof we show that this is actually the case. 

 First, to simplify the notations, let 
 \[
g(\bpii_s,\bq_s) = \bbE_{\bpii_s}\Bigg[r(a|s) - \ln \pi(a|s) + \gamma \bbE_{\bq_s}[V^*(s')] \Bigg],\ v(a,s|\bq_s) = \exp\left(r(a|s) +  \gamma \bbE_{\bq_s}\Big[V^*(s')\Big]\right).
\]
We see that  $(\bq^*_s,\bpii^*_s)$ specified in Theorem \ref{theor:optimal-policies-s-rect} is an optimal solution to the minimax problem $\min_{\bq_s \in\cQ_s}\max_{\bpii_s} \{g(\bpii_s,\bq_s)\}$.
We will show that $(\bq^*_s,\bpii^*_s)$ is also a saddle point of the minimax problem, thus also a solution to the \textit{max-min} counterpart. From the definition of $(\bq^*_s,\bpii^*_s)$ we have
$\bpii^*_s = \text{argmax}_{\bpii_s} g(\bpii_s,\bq^*_s)$. Now we prove that $\bq^*_s = \text{argmin}_{\bq_s} g(\bpii^*_s,\bq_s)$. From the definition of the optimal policy $\bpii^*_s$  we write
\begin{align}
    g(\bpii^*_s,\bq_s) = \sum_{a}\frac{v(a,s|\bq^*_s) \left(r(a|s) -\ln \pi^*(a|s)\right)}{\sum_{a'} v(a',s|\bq^*_s)} + \frac{\gamma \sum_a \sum_{s'} v(a,s|\bq^*_s)q(s'|a,s)V(s') }{\sum_{a'} v(a',s|\bq^*_s)}. \label{rect-s-optimal-policy-proof-eq0}
\end{align}
Recall that $\bq^*_s$ is an optimal solution to $\min_{\bq_s}\{h(\bq_s) = \sum_{a} v(a,s|\bq_s)\}$. 
Now, consider any point $\bq'_s\in\cQ_s$ and denote $\bdt_s = \bq'_s -\bq^*_s$. The convexity of $\cQ_s$ implies that $\bq^*_s+ \alpha \bdt_s \in\cQ$, for any $\alpha\in [0,1]$, and there exists $\beta\in [0,1]$ such that
\[
h(\bq^*_s + \alpha \bdt_s) - h(\bq^*_s) = \nabla_{\bq_s}h(\bq^*_s + \alpha\beta \bdt)^{\T}(\alpha \bdt_s), 
\]
where the equality is due to the fact that $h(\bq_s)$ is differentiable and  convex. This is also equivalent to
\begin{equation}
\label{eq:rect-s-optimal-policy--proof-eq1}
\nabla_{\bq_s}h(\bq^*_s + \alpha\beta \bdt)^\T \bdt_s = \frac{h(\bq^*_s + \alpha \bdt_s) - h(\bq^*_s)}{\alpha}. 
\end{equation}
Let $\alpha \rightarrow 0$, the left side of \eqref{eq:rect-s-optimal-policy--proof-eq1} converges to $\nabla_{\bq_s}h(\bq^*_s)^\T \bdt_s$ while the right side is always non-negative. As a result, we need to have $\nabla_{\bq_s}h(\bq^*_s)^\T \bdt_s \geq 0$. To show this more precisely, assume that $\nabla_{\bq_s}h(\bq^*_s)^\T \bdt_s < 0$, then the continuity of the left side of \eqref{eq:rect-s-optimal-policy--proof-eq1} implies that  there exists $\alpha$ small enough such that  $\nabla_{\bq_s}h(\bq^*_s + \alpha\beta \bdt)^\T \bdt_s<0$, which means $h(\bq^*_s + \alpha \bdt_s) < h(\bq^*_s)$, which is contrary to the definition of $\bq^*_s$. So, we have   $\nabla_{\bq_s}h(\bq^*_s)^\T \bdt_s \geq  0$ or   $\nabla_{\bq_s}h(\bq^*_s)^\T \bq'_s \geq \nabla_{\bq_s}h(\bq^*_s)^\T \bq^*_s$. Since we can choose $\bq'_s$ arbitrarily in $\cQ$, we have $\bq^*_s = \text{argmin}_{\bq_s} \nabla_{\bq_s}h(\bq^*_s)^\T \bq_s$. Moreover, 
\begin{align}
    \nabla_{\bq_s}h(\bq^*_s)^\T \bq_s = \gamma \sum_{a} v(a,s|\bq_s^*)\sum_{s'} V(s') q(s'|a,s). \label{rect-s-optimal-policy-proof-eq2}
\end{align}
Combine \eqref{rect-s-optimal-policy-proof-eq0}
 and \eqref{rect-s-optimal-policy-proof-eq2} and the recent claim that $\bq^*_s = \text{argmin}_{\bq_s} \nabla_{\bq_s}h(\bq^*_s)^\T \bq_s$, we have $\bq^*_s = \text{argmin}_{\bq_s} g(\bpii^*_s,\bq_s)$. As such, 
 $(\bq^*_s,\bpii^*_s)$ is also a saddle point of the minimax problem $\min_{\bq_s \in\cQ_s}\max_{\bpii_s} \{g(\bpii_s,\bq_s)\}$. We need one more step to prove that the policies determined in the theorem is optimal to the max-min problem $\max_{\bpii_s} \min_{\bq_s \in\cQ_s}\{g(\bpii_s,\bq_s)\}$. Using the property of the saddle point, for any policies $\bpii_s$ we have
 \begin{align}
     \min_{\bq_s \in\cQ_s}\{g(\bpii_s,\bq_s)\}&\leq g(\bpii_s,\bq^*_s) \leq \max_{\bpii_s} \{g(\bpii_s,\bq^*_s)\}\nonumber\\
      &= g(\bpii^*_s,\bq^*_s) = \min_{\bq_s}\{g(\bpii^*_s,\bq_s)\},\nonumber
 \end{align}
 which finally implies that $(\bpii^*,\bq^*)$ determined in the theorem is also optimal to the original \textit{max-min} problem. we complete the proof.   

\subsection{Proof of Theorem \ref{th:ER-approx-CM}}
	We first prove the following result. Let define a function $f:\bbR^I\rightarrow \bbR$ as $f(\bx) = \ln\left(\sum_{i=1}^I e^{x_i}\right)$. Given any vector $\bx,\by\in \bbR^{I}$, the mean value theorem implies that there is $\bz = \alpha \bx+ (1-\alpha)\by$, $\alpha \in [0,1]$, such that
	\begin{equation}
	\label{eq:proof-th-approx-eq1}
	|f(\bx) - f(\by)| = |\nabla f(\bz)^\T (\bx-\by) | \leq \sum_{i\in[I]}||\bx-\by||_\infty \frac{e^{z_i}}{\sum_{i'\in[I]}e^{z_{i'}}} = ||\bx-\by||_{\infty}.
	\end{equation}
	
\textbf{For (i)},  we  first have, for any $s\in \cS$, we have
\[
{\cT}[V](s) \leq  \widetilde{\cT}[V](s)  \leq  \eta\ln\left(\sum_{a\in \cA}\exp\Big(r(a|s)/\eta +  \gamma  \min_{\bq_{sa}} \left\{\sum_{s'\in\cS} {q}(s'|a,s)V(s')\right\}/\eta + \gamma\xi/\eta \Big) \right).
\]
So, the inequality in \eqref{eq:proof-th-approx-eq1} tells us that
{\small \begin{align}
    |{\cT}[V](s)  - \widetilde{\cT}[V](s)|&\leq \left|{\cT}[V](s) -\eta \ln\left(\sum_{a\in \cA}\exp\Big(r(a|s)/\eta +  \gamma  \min_{\bq_{sa}} \left\{\sum_{s'\in\cS} {q}(s'|a,s)V(s')\right\}\eta + \gamma\xi/\eta \Big) \right)  \right| \nonumber \\
    \leq \gamma\xi, \nonumber
\end{align}}
which  means 
\begin{equation}
\label{eq:proof-th-approx-eq2}
||{\cT}[V]  - \widetilde{\cT}[V]||_\infty \leq \gamma \xi.    
\end{equation}
Moreover, using the triangle inequality, We can further write
\begin{align}
    ||{\cT}^n[V]  - \widetilde{\cT}^n[V]||_\infty &\leq  ||\widetilde{\cT}^n[V] -\cT[\widetilde{\cT}^{n-1}[V]]||_\infty +  ||\cT[\widetilde{\cT}^{n-1}[V]] - \cT^n[V]||_\infty\nonumber \\
    &\stackrel{(*)}{\leq}  \gamma \xi + \gamma ||{\cT}^{n-1}[V]  - \widetilde{\cT}^{n-1}[V] ||_\infty  \nonumber\\
    &\leq  \ldots \nonumber\\
    & \leq \gamma\xi (1+\ldots+\gamma^{[n-1]})\nonumber\\
    & = \gamma\xi (1-\gamma^{[n]})/(1-\gamma),\nonumber
\end{align}
where (*) is due to \eqref{eq:proof-th-approx-eq2}. This is the desired bound.

\textbf{For (ii),} we need the following chain of claims.
\begin{itemize}
    \item{\textbf{Claim 1:}} For any $V\in\bbR^{|\cS|}$
\begin{align}
||\cT[V]-V||_\infty \leq ||\widetilde{\cT}[V]-\cT[V]||_\infty + ||\widetilde{\cT}[V]-V||_\infty \leq \gamma \xi + ||\widetilde{\cT}[V]-V||_\infty 
\end{align}
\item{\textbf{Claim 2:}} For any $V\in\bbR^{|\cS|}$
\begin{align}
||\cT^n[V]-V||_\infty &\leq ||\cT^{n}[V]-\cT^{n-1}[V]||_\infty + ||\cT^{n-1}[V]-V||_\infty \nonumber \\
 &\leq  \gamma^{[n-1]}||\cT[V] - V||_\infty + ||\cT^{n-1}[V]-V||_\infty \nonumber \\
 &\leq ||\cT[V] - V||_\infty \left(1+\ldots+\gamma^{[n-1]}\right)\nonumber\\
 &=||\cT[V] - V||_\infty\frac{1-\gamma^{[n]}}{1-\gamma}.\nonumber
\end{align}
So
\[
||V - V^*||_\infty \leq ||\cT[V] - V||_\infty\frac{1}{1-\gamma}.
\]
\item{\textbf{Claim 3:}} For any $V\in\bbR^{|\cS|}$ and $n\in \mathbb{N}_+$
\begin{align}
||\widetilde{\cT}^n[V]-\widetilde{\cT}^{n-1}[V]||_\infty & \leq  ||\widetilde{\cT}^n[V]-\cT[\widetilde{\cT}^{n-1}[V]]||_\infty + ||\widetilde{\cT}^{n-1}[V]-\cT[\widetilde{\cT}^{n-2}[V]]||_\infty + \nonumber\\
&\qquad \gamma ||\widetilde{\cT}^{n-1}[V]-\widetilde{\cT}^{n-2}[V]||_\infty \nonumber \\
&\stackrel{(**)}{\leq} 2\gamma\xi + \gamma ||\widetilde{\cT}^{n-1}[V]-\widetilde{\cT}^{n-2}[V]||_\infty, \nonumber
\end{align}
where (**) is due to \eqref{eq:proof-th-approx-eq2}.
So,
\[
||\widetilde{\cT}^n[V]-\widetilde{\cT}^{n-1}[V]||_\infty \leq 2\xi\gamma\frac{1-\gamma^{[n-1]}}{1-\gamma}+\gamma^{[n-1]}||\widetilde{\cT}[V]-V||_\infty.
\]
\item{\textbf{Claim 4:}} For any $V\in\bbR^{|\cS|}$ and $n\in \mathbb{N}_+$
\begin{align}
||\widetilde{\cT}^n[V]-V^*||_\infty &\leq \frac{1}{1-\gamma}||\cT[\widetilde{\cT}^n[V]]-\widetilde{\cT}^n[V]||_\infty   \nonumber \\
&\leq \frac{\gamma\xi}{1-\gamma} + \frac{1}{1-\gamma}||\widetilde{\cT}^{n+1}[V]-\widetilde{\cT}^n[V]||_\infty   \nonumber \\
&\stackrel{(***)}{\leq} \frac{\gamma\xi}{1-\gamma} + \frac{2\gamma\xi}{(1-\gamma)^2} + \frac{\gamma^{[n]}}{1-\gamma} ||\widetilde{\cT}[V]-V||_\infty,
\end{align}
where (***) is due to \textbf{Claim 4}.
\end{itemize}

Thus, to have $||\widetilde{\cT}^n[V]-V^*||_\infty\leq \epsilon$, it is necessary to select $\xi\leq \epsilon (1-\gamma)^2/(4\gamma)$ and $||\widetilde{\cT}^{n+1}[V]-\widetilde{\cT}^n[V]||_\infty \leq 3\epsilon(1-\gamma)/4$. Note that the latter inequality always occurs if $n$ is large enough, because 
\begin{align}
||\widetilde{\cT}^{n+1}[V]-\widetilde{\cT}^n[V]||_\infty &\leq 2\xi\gamma\frac{1}{1-\gamma}+\gamma^{[n-1]}||\widetilde{\cT}[V]-V||_\infty \nonumber\\
&\leq \epsilon(1-\gamma)/2+\gamma^{[n-1]}||\widetilde{\cT}[V]-V||_\infty \nonumber
\end{align}
and the term $\gamma^{[n-1]}||\widetilde{\cT}[V]-V||_\infty$ converges to zero when $n\rightarrow \infty$. Moreover, we see that it would requires $n = \cO(\ln \epsilon^{-1})$ to have $\gamma^{[n-1]}||\widetilde{\cT}[V]-V||_\infty \leq \epsilon(1-\gamma)/4$.

For the last claim \textbf{(iii)}
, we write the optimal policy and the approximate policy as
\[
\pi^*(a|s) =\frac{\exp(z(a,s|V^*,\bq^*))/\eta}{\sum_{a'}\exp(z(a',s|V^*,\bq^*)/\eta)}\text{; and }\widetilde{\pi}(a|s) =\frac{\exp(z(a,s|\widetilde{V},\overline{\bq}))/\eta}{\sum_{a'}\exp(z(a',s|\widetilde{V},\overline{\bq})/\eta)}
\]
where $z(a,s|\widetilde{V},\overline{\bq}) = r(a|s) + \gamma \bbE_{\bar{\bq}_{sa}}[\widetilde{V}(s')]$. We see that, for any $a\in\cA, s\in \cS$
\begin{align}
|z(a,s|V^*,\bq^*) - z(a,s|\widetilde{V},\overline{\bq})|&= \gamma |\bbE_{\bar{\bq}_{sa}}[\widetilde{V}(s')] - \min_{\bq_{sa}} \bbE_{\bq_{sa}}[V^*(s')]| \nonumber\\
&\leq  \gamma |\bbE_{\bar{\bq}_{sa}}[\widetilde{V}(s')] - \min_{\bq_{sa}} \bbE_{\bq_{sa}}[\widetilde{V}(s')]|  + \gamma |\min_{\bq_{sa}} \bbE_{\bq_{sa}}[\widetilde{V}(s')] \nonumber \\
& \qquad - \min_{\bq_{sa}} \bbE_{\bq_{sa}}[V^*(s')]| \nonumber \\
&\stackrel{(i)}{\leq}  \gamma \xi +  \gamma |\min_{\bq_{sa}} \bbE_{\bq_{sa}}[\widetilde{V}(s')] - \min_{\bq_{sa}} \bbE_{\bq_{sa}}[V^*(s')]|\label{eq:proofx-eq0}
\end{align}
We now consider two cases
\begin{itemize}
    \item If $\min_{\bq_{sa}} \{\bbE_{\bq_{sa}}[\widetilde{V}(s')]\} \geq  \min_{\bq_{sa}} \{\bbE_{\bq_{sa}}[V^*(s')]\}$, then let $\bq^*_{sa}$ be a solution attaining the optimal value $ \min_{\bq_{sa}} \{\bbE_{\bq_{sa}}[V^*(s')]\}$. We have
    \begin{align}
    \min_{\bq_{sa}} \bbE_{\bq_{sa}}[\widetilde{V}(s')] - \min_{\bq_{sa}} \bbE_{\bq_{sa}}[V^*(s')] &\leq |\bbE_{\bq^*_{sa}} [\widetilde{V}(s')-V^*(s')]|\nonumber \\
    &\leq ||\widetilde{V}-{V^*}||_\infty \label{eq:proofx-eq1}
    \end{align}
    \item If $\min_{\bq_{sa}} \{\bbE_{\bq_{sa}}[\widetilde{V}(s')]\} <  \min_{\bq_{sa}} \{\bbE_{\bq_{sa}}[V^*(s')]\}$, then similarly we let $\bq^*_{sa}$ be a solution attaining the optimal value $ \min_{\bq_{sa}} \{\bbE_{\bq_{sa}}[\widehat{V}(s')]\}$ and obtain
    \begin{equation}
    \label{eq:proofx-eq2}
         \min_{\bq_{sa}} \bbE_{\bq_{sa}}[\widetilde{V}(s')] - \min_{\bq_{sa}} \bbE_{\bq_{sa}}[V^*(s')] \leq ||\widetilde{V}-{V^*}||_\infty
    \end{equation}
\end{itemize}
Combine \eqref{eq:proofx-eq0}, \eqref{eq:proofx-eq1} and \eqref{eq:proofx-eq2} we have
\begin{equation}
\label{eq:proofx-eq3}
|z(a,s|V^*,\bq^*) - z(a,s|\widetilde{V},\overline{\bq})|\leq \gamma (\xi+\epsilon).
\end{equation}
We now look at the difference between $\widetilde{\pi}(a|s)$ and $\pi^*(a|s)$ as
\begin{align}
&\left|\ln\frac{\widetilde{\pi}(a|s)}{\pi^*(a|s)}\right|\nonumber\\
&\leq \frac{1}{\eta}|z(a,s|V^*,\bq^*) - z(a,s|\widetilde{V},\overline{\bq})| + \nonumber \\
& \qquad\left|\ln \left(\sum_{a'}\exp(z(a,s|\widetilde{V},\overline{\bq})/\eta) \right) - \ln \left(\sum_{a'}\exp(z(a,s|V^*,\bq^*)/\eta) \right)  \right|\nonumber \\
&\stackrel{(i)}{\leq} \frac{1}{\eta}|z(a,s|V^*,\bq^*) - z(a,s|\widetilde{V},\overline{\bq})| + \frac{1}{\eta}\max_{a'} |z(a',s|V^*,\bq^*)- z(a',s|\widetilde{V},\overline{\bq})|\nonumber \\
&\stackrel{(ii)}{\leq} \frac{2}{\eta}(\xi+\epsilon),\label{eq:proofx-eq4}
\end{align}
where (i) is due to \eqref{eq:proof-th-approx-eq1} and (ii) is due to \eqref{eq:proofx-eq3}. Continue to elaborate \eqref{eq:proofx-eq4} we get
\begin{align}
    \frac{|\widetilde{\pi}(a|s)-\pi^*(a|s)|}{\pi^*(a|s)}\leq \exp(2(\xi+\epsilon)/\eta) - 1,
\end{align}
thus $|\widetilde{\pi}(a|s)-\pi^*(a|s)|\leq \exp(2(\xi+\epsilon)/\eta) - 1$, which leads to the desired bound.

\section{Relevant Algorithms and Discussions}

\subsection{Hardness of Solving Robust ER-MDP with Non-rectangular Uncertainty Sets}
Theorems \ref{th:main-sa-rec} and \ref{theor:optimal-policies-s-rect} imply that if we can efficiently (i.e., in polynomial time) solve the inner minimization problems $\min_{\bq_{sa}}\bbE[V(s')]$ in the $(s,a)$-rectangularity case and $\min_{\bq_s\in\cQ_s} \{\sum_{a\in \cA}\exp(z(a,s|V^*,\bq) )\}$ in the $(s,a)$-rectangularity case, then we can compute the value function as well as the optimal policy in polynomial time as well. We will discuss this in the next question. A relevant question here is that what happens if the uncertainty set is not rectangular. \cite{Wiesemann2013robustMDP} consider the standard roust MDP and  show that, if this is the case, then  unless $\cP = \cN\cP$, it is generally not possible achieve an $\epsilon$-approximation of the expected accumulated  reward in polynomial time. We can show that this result also holds for the robust entropy-regularized MDP model. Our argument is that, for any $\eta>0$, if there is an algorithm $\bX$ that is able to give a $\epsilon$-approximation of $\max_{\bpii}\min_{\bq}F^{\eta}_{\infty}(\br,\bpii,\bq)$, for any $\epsilon>0$, where $F^{\eta}_{\infty}(\bpii,\bq)$ is the expected accumulated regularized reward as in \eqref{prob:DET-Max-EP} but we add  $\eta$ and the reward function $\br$  as parameters to facilitate our arguments. Now, for any $N>0$ we can solve the regularized problem (in polynomial time) by $\bX$ with rewards $\br' = \br \times N$ and approximation error $\epsilon$, then we see that the algorithm will give an $(\epsilon/2)$-approximation of $\big(N\max_{\bpii}\min_{\bq}F^{\eta/N}_{\infty}(\br,\bpii,\bq)\big)$, or a $(\epsilon/2)$-approximation of $\big(\max_{\bpii}\min_{\bq}F^{\eta/N}_{\infty}(\br,\bpii,\bq)\big)$. Furthermore, by choosing $N$ large enough, we also have $|\max_{\bpii}\min_{\bq}F^{\eta/N}_{\infty}(\br,\bpii,\bq) -   \max_{\bpii}\min_{\bq}F^{0}_{\infty}(\br,\bpii,\bq)|<\epsilon/2$, noting that $F^{0}_{\infty}(\br,\bpii,\bq)$ is the objective in the unregularized case. By a triangle inequality, we see that $|\widetilde{F} - \max_{\bpii}\min_{\bq}F^{0}_{\infty}(\br,\bpii,\bq)|<\epsilon$, which means that Algorithm $\bX$ can give a $\epsilon$-approximation of $\max_{\bpii}\min_{\bq}F^{0}_{\infty}(\br,\bpii,\bq)$, which contradicts \cite{Wiesemann2013robustMDP}'s claims.  

\subsection{Approximate Robust Value Iteration}

\begin{algorithm}[H]
	\caption{Robust value iteration}
	\label{algo:VI-infinite}
	\begin{algorithmic}
		\STATE \comments{Compute an $\epsilon$-approximation of  $V^*$}
		\STATE $V = V^0 = \textbf{0}$, $\overline{V} = \textbf{1}$\\
		\REPEAT
		\STATE $\overline{V} = V;$
		\STATE Solve the inner minimization problem $\min_{\bq_{as}} \bbE_{\bq_{sa}}[V(s)]$ by a $\xi$-approximation algorithm, where
        $
		\xi = \epsilon(1-\gamma)^2/(4\gamma).
		$
		Then, update $V\leftarrow \widetilde{\cT}[V]$.
		\UNTIL{$||V-\overline{V}||_\infty \leq 3\epsilon(1-\gamma)/4$.}	
			\STATE \comments{Compute an $\epsilon$-approximation of  $\bpii^*$}
		\STATE $V = V^0 = \textbf{0}$, $\overline{V} = \textbf{1}$\\
		\REPEAT
		\STATE $\overline{V} = V;$
		\STATE Solve the inner minimization problem $\min_{\bq_{as}} \bbE_{\bq_{sa}}[V(s)]$ by a $\xi$-approximation algorithm, where
		$
		\xi = \ln(\epsilon+1)(1-\gamma)^2/(8\gamma).
		$
		Then, update $V\leftarrow \widetilde{\cT}[V]$.
		\UNTIL{$||V-\overline{V}||_\infty \leq 3\ln(\epsilon+1)(1-\gamma)/8$.}	
	\end{algorithmic}
\end{algorithm}

\subsection{Robust IRL}
 In perspective of imitation learning/IRL, we are interested  in approximating the log-likelihood function. 
Assume that the demonstrated data consists of $I$ trajectories and $i$-th trajectory contains $K_i$ state-action observations.  
 The average log-likelihood function is defined as
 $
 \cL(\Omega|\theta) = \frac{1}{I}\sum_{i=1}^I \cL(\omega_i|\theta) 
 $, where $\theta$ is a vector of parameters to be inferred from the data,  and $\cL(\omega_i|\theta)$ is the log-likelihood value of sequence $\omega_i$, $i=1,\ldots,I$, defined as 
 $
 \cL(\omega_i|\theta) = \sum_{t = 0}^{K_i-1} \ln \pi^*(a^i_t|s^i_t).
 $
 The following algorithm describe a robust IRL algorithm that allows to learn from \textit{conservative} behavior. 
\begin{algorithm}[H]
\caption{Robust infinite-horizon IRL}
\label{algo:MLE-infinite}
\begin{algorithmic}
	\STATE \comments{Compute an $\epsilon$-approximation of $\cL(\Omega|\theta)$}
 \STATE $V = V^0 = \textbf{0}$, $\overline{V} = \textbf{1}$\\
 \FOR{each sequence $\omega_i$, $i\in [I]$} 
 \REPEAT
  \STATE $\overline{V} = V;$
  \STATE Solve $\min_{\bq_{as}} \bbE_{\bq_{sa}}[V(s)]$ by a $\xi$-approximation algorithm. 
  where
  $\xi = \frac{\epsilon(1-\gamma)^2}{8\gamma^2\max_i\{K_i\}}.$
  \STATE Update $V\leftarrow \widetilde{\cT}[V]$.
 \UNTIL{$||V-\overline{V}||_\infty \leq 3\epsilon(1-\gamma)/(8\gamma \max_i \{K_i\})$}
 \STATE Compute $P(a^i_k|s^i_k,\theta)$, $k=0,\ldots,K_i$, based on fixed point solution $V$, and $P(\omega_i|\theta) =\prod_{k=0}^{K_i} P(a^i_k|s^i_k,\theta)$
 \ENDFOR
\STATE Return $\widetilde{\cL}(\Omega|\theta) = 1/I\sum_{i\in [I]} \ln \bbP(\omega_i|\theta) $
\end{algorithmic}
\end{algorithm}

Similarly to the finite case,   we can show that if we can compute an $\epsilon^V$-approximation of the fixed point solution $V^*$, then we can achieve a $(2\gamma\epsilon^V\max_i\{K_i\})$-approximation of  $\cL(\omega_i|\theta)$ and $\cL(\Omega|\theta)$.
Algorithm \ref{algo:MLE-infinite} presents main steps to compute  an $\epsilon$-approximation of the log-likelihood function. 
The computational complexity in the case of single KL divergence bound is $\cO\left(I|\cS||\cA| \max_s \{N_s\}(\ln \epsilon^{-1})^2 \right)$ and in the case of several bounds with interior-point algorithms, the worst-case complexity  is  $\cO\left(I|\cS||\cA| (\max_s \{N_s\})^{7/2}(\ln \epsilon^{-1})^2 \right)$.  On the other hand, when the transition probabilities are assumed to be known with certainty, this worst-case complexity becomes $\cO\left(I|\cS||\cA| (\max_s \{N_s\})\ln \epsilon^{-1} \right)$.

%

\subsection{Prediction Log-loss Guarantee in Robust ER-MDP}
It is also interesting to look at how our robust ER-MDP model is connected to the standard maximum causal entropy principle \citep{ziebart2010_IRL_Causal}.
Proposition \ref{prop:log-loss-infinite} below shows that, in analogy to \cite{ziebart2010_IRL_Causal}, the prediction log-loss guarantee holds for the robust ER-MDP model, but with an additional level of  robustness w.r.t uncertain dynamics $\bQ$. This result is also relevant to a claim in \cite{Eysenbach2019if} saying that the standard ER-MDP  is robust for a  certain class of reward functions. This implies that our robust ER-MDP model adds another level of robustness to their setting when the dynamics are ambiguous.  
\begin{proposition}
\label{prop:log-loss-infinite}	
Assume that $\bQ$ is  $(a,s)$-rectangular and compact, or  $(a,s)$-rectangular, compact and convex, let  $(\bpii^*,\bq^*)$ be the solution determined in Theorems \ref{th:main-sa-rec} or ~\ref{theor:optimal-policies-s-rect}, then  $\pmb{\rho}^t = \bpii^*$ and $\bq^t = \bq^*$, for $t= 0,\ldots,\infty$, minimize the prediction log-loss
$$
\min_{\substack{\pmb{\rho}^t \in \Delta^\pi \\ \bq^t\in\cQ  \\ t=0,...}}
\max_{\substack{\bpii^t \in \Delta^{\pi},\; t=0,...\\\bbE_{\tau\sim (\bPi,\bQ)}[R(\tau)] = \widetilde{\bbE}^R}} \bbE_{\bPi,\bQ}\left[ \sum_{t=0}^{\infty} -\gamma^{[t]}\ln \rho^t(a_t|s_t)\right],
$$
where $R(\tau)$ is the accumulated and discounted reward of trajectory $\tau = \{(s_0,a_0), (s_1,a_1), ....\}$ and $\widetilde{\bbE}^R$ is empirical expectation of the accumulated rewards.  
\end{proposition}

\begin{proof}

Under the assumptions, we see that $(\bq^t,\bpii^t) = (\bq^*, \bpii^*)$, $t = 0,\ldots$, is a solution to the problem
\begin{equation}
\label{eq:proof-coro32-eq1}
\min_{\bq^0,\ldots} \max_{\bpii^0,\ldots} \left\{\bbE_{\bQ,\bPi}\left[\sum_{t=0}^\infty \gamma^{[t]} \Big(r(a_t|s_t) - \eta\ln \pi^t(a_t|s_t)\Big) \right]\right\}.
\end{equation}
It is also well-known that the inner maximization optimization problem of \eqref{eq:proof-coro32-eq1} can be formulated equivalently as a maximum causal entropy problem \citep{ziebart2010_IRL_Causal}
\begin{align}
\underset{\bpii^0,\ldots}{\text{sup}}\qquad & \bbE_{\bpii^0,\ldots} \left[\sum_{t=0}^{\infty}- \gamma^{[t]}\eta \ln \pi^t(a_t|s_t) \right] & \label{eq:proof-coro32-eq2}\\
\text{subject to} \qquad & \bbE_{\tau \sim (\bPi,\bQ)} \left[R(\tau) \right]  = \widetilde{\bbE}^R.&  \nonumber\\
&\bpii^t \in\Delta^\pi,\ t=0,\ldots&\nonumber
\end{align}
The prediction log-loss guarantee shown in \cite{ziebart2010_IRL_Causal} also implies that if $\bpii^*$ is an optimal solution to \eqref{eq:proof-coro32-eq2}, then it is also a solution to the problem
\begin{align}
\underset{\substack{\pmb{\rho}^t \in \Delta^\pi \\ t=0,\ldots}}{\text{min}} \underset{\substack{\bpii^t \in \Delta^\pi\\t=0,\ldots}}{\text{max}}\qquad & \bbE_{\bPi,\bQ} \left[\sum_{t=0}^{T}-\eta \ln \rho(a_t|s_t) \right] & \label{eq:proof-coro32-eq3}\\
\text{subject to} \qquad & \bbE_{\tau \sim (\bPi,\bQ)} \left[R(\tau) \right]  = \widetilde{\bbE}^R.&  \nonumber
\end{align}
Combine \eqref{eq:proof-coro32-eq1}, \eqref{eq:proof-coro32-eq2} and \eqref{eq:proof-coro32-eq3} we obtain the desired result.
\end{proof}

%

\subsection{Robust General-Regularized MDP}
\label{sec:extension-rMDPs}
We show how our results can be extended to the general regularized MDP framework introduced in \cite{Geist2019theory}. In a regularized model, a regularized term $\phi_{s}(\bpii_{s})$ are added to the reward \cite{Geist2019theory}, for any $s\in \cS$. The Markov decision problem in the finite-horizon case can be stated as
\begin{align}
\max_{\substack{\bpii^t \in \Delta^\pi \\t=0,\ldots }} \min_{\substack{\bq^t\in\cQ\\t=0,\ldots}}\Bigg\{ \bbE_{\tau \sim (\bPi,\bQ)}\Bigg[\sum_{t=0}^{\infty}\gamma^{[t]} \Big( r(a_t|s_t) + \phi_{s_t}(\bpii^t_{s_t}) \Big)\Bigg]\Bigg\}.\label{prob:Regularized-MDPs}
\end{align}
It is typically assumed that $\phi_s(\bpii_s)$ is concave and bounded. 
If $\phi_s(\bpii_s) = -\sum_{a\in\cA} \pi(a|s)\ln \pi(a|s)$ (negative relative entropy), the model becomes the ER-MDP model studied above. In analogy to the ER-MDP, we define the mapping $\cT^{\bphi}[V]:\bbR^{|\cS|}\rightarrow \bbR^{|\cS|}$
\[
\cT^{\bphi}[V](s) = \max_{\bpii_s\in\Delta^{\pi}_s}\min_{\bq_s\in\cQ_s}\left\{ \bbE_{\bpii_s,\bq_s}\Big[ r(a|s)+ \gamma\sum_{s'}q(s'|s,a)V(s') \Big]  + \phi_s(\bpii_s) \right\}.
\]
Then the contraction mapping can be verified  analogously as in the entropy-regularized models. That is,  under both $(s,a)$ and $(s)$-rectangularity conditions, i.e., for any  $V,V'\in\bbR^{|\cS|}$, we have 
	$\cT^{\bphi}[V]$ is a contraction mapping with parameter $\gamma$ $||\cT^{\bphi}[V]-\cT^{\bphi}[V']||_\infty \leq \gamma||V-V'||_\infty,$ 
 with a note that the contraction property for the non-robust model has been shown in \cite{Geist2019theory}. The other basic properties, except the perfect duality can be proved similarly as well. For example, we also have the Markov optimality saying that, under the two uncertainty assumptions, 
if $V^{\bphi,*}$ is a solution to the equation $\cT^{\bphi}[V]=  V$, then 
\[
V^{\bphi,*}(s) = \max_{\substack{\bpii^t \in \Delta^\pi \\t=0,\ldots }} \min_{\substack{\bq^t\in\cQ\\t=0,\ldots}}\Bigg\{ \bbE_{\tau \sim (\bPi,\bQ)}\Bigg[\sum_{t=0}^{\infty}\gamma^{[t]} \Big( r(a_t|s_t) + \phi_{s_t}(\bpii^t_{s_t}) \Big)\ \Bigg|s_0=s\Bigg]\Bigg\},
\]
leading to the result that one can solve the Bellman equation $\cT^{\phi}[V] = V$ to get an optimal policy, as 
	\[
	\pi^*_s = \text{argmax}_{\bpii_s}\left\{ \bw_s^\T \bpii_s + \phi_s(\bpii_s)\right\}
	\]
	where $\bw_s \in\bbR^{|\cA|}$ with entries
	\[
	w_{sa} = r(a|s)+ \gamma \min_{\bq_s\in\cQ_s}\left\{\sum_{s'} q(s'|s,a)V^{\bphi,*}(s') \right\}.
	\]
Note that if the convex conjugate  function (i.e. Legendre-Fenchel transform)  of $-\phi_s(\bpii_s)$ can be computed efficiently, then the contraction mapping $\cT^{\bphi}[V]$ can be expressed as
$
\cT^{\bphi}[V]  = \phi^*_s(\bw_s),
$ 
where $\phi^*_s(\bw_s)$ is the convex conjugate of$-\phi_s(\bpii_s)$ in $\Delta^\pi_s$, and $\bw_s\in\bbR^{|\cA|}$ with entries $	w_{sa} = r(a|s)+ \gamma \min_{\bq_s\in\cQ_s}\left\{\sum_{s'} q(s'|s,a)V(s') \right\}$. 
 
If the uncertainty set $\cQ$ is only $(s)$-rectangular, the inner infimum problem in the mapping $\cT^{\bphi}[V]$ involves $\bpii_s$ as decision variables.  Thus,  solving the robust Bellman equation is more difficult. However, these \textit{max-min} problems can be solved efficiently by a saddle point algorithm, e.g., the Frank-Wolfe algorithms proposed in  \cite{Gidel2016frank}. In this context, the computational complexity is more difficult to bound, as compared to what we have in Section \ref{sec:uncertainty-model}.

We also can show that the perfect duality also holds in the context, for any uncertainty set if $\cQ$ is $(s,a)$-rectangular and for convex and compact  if $\cQ$ is $(s)$-rectangular. The proof can be done by showing that the perfect duality holds for the robust Bellman equation using the  \textit{von Neumann's minimax } theorem, analogously to the entropy-regularized case.
The perfect duality property would be helpful for solving the robust Bellman equation in the $(s)$-rectangularity case. More precisely, in case that the convex conjugate of $-\phi_s(\bpii_s)$ can be computed conveniently (e.g., by an analytical form), one can solve the min-max counterpart of $\cT^{\bphi}[V]$ as 
\[
\widetilde{\cT}^{\bphi}[V](s) = \min_{\bq_s\in \cQ_s} \max_{\bpii_s\in\Delta^\pi_s}\left\{ \bpii_s^\T \bw_s(\bq_s) + \phi_s(\bpii_s) \right\} = \min_{\bq_s} \phi_s^*(\bw_s(\bq_s)),
\]
 where $\bw_s(\bq_s) \in\bbR^{|\cA|}$ with entries 
 $
 w_{sa}(\bq_s) = r(a|s)+ \gamma \sum_{s'} q(s'|s,a)V(s'). 
 $
 Since  $\bw_s(\bq_s)$ is linear in $\bq_s$, $ \phi_s^*(\bw_s(\bq_s))$ is concave in $\bq_s$, which implies that the problem  $\min_{\bq_s} \phi_s^*(\bw_s(\bq_s))$ can be solved efficiently in polynomial time by a convex optimization algorithm (e.g., interior-point). Recall that in the entropy-regularized model, the convex conjugate function of $\phi_s(\cdot)$ has the closed form $\phi^*_s(\bw_s) = \ln\left(\sum_{a\in\cA} \exp(w_{sa})\right)$. In the general regularized case, there might be no closed form to compute $\phi^*_s(\bw_s)$ and one might need to do it approximately. This would lead to an additional approximation error in the error propagation of the approximate value iteration or (modified) policy iteration.

\subsection{Complexity Analyses for the Adversary's Problems}
\label{proof:Complexity}
We analyze the computational complexity of solving the adversary's problem, under two rectangularity settings and with uncertainty sets based on several KL bounds. 	

\subsubsection{{$(s,a)$-rectangularity}}
First, for notational simplicity, we consider a compact version of the inner optimization problem  $\min_{\bx}\left\{ \sum_{i=1}^{N_s}x_ic_i|\ \bx\in \cX \subset \Delta(N_s)\right\}$, where $\Delta(|N_s|)$ is the simplex in $\bbR^{N_s}$ and, 
 $N_s$ is the number of states that can be reached from $s$. Normally, $N_s \ll |\cS|$. In a likelihood model, the uncertainty set has the form $\cX = \{\bx \in \Delta(N_s)|\ \sum_{i} \hat{x}_i\ln x_i \geq \beta\}$, where  $\hat{x}_i$ is an empirical estimate and 
 $\beta$ is a scalar representing an uncertainty level, such that $\sum_{i} \hat{x}_i\ln \hat{x}_i \geq \beta$. On the other hand, a relative entropy model takes the form $\cX = \{\bx \in\Delta(|N_s|)|\ \sum_{i} x_i\ln(x_i/\hat{x}_i) \leq \beta\}$. It is possible to show that under the above uncertainty sets, one can achieve a $\xi$-approximation of the inner optimal value by bisection, with complexity $C(\xi) = \cO(N_s\ln\xi^{-1})$, for any $s\in \cS$. We refer the reader to \cite{Nilim2004robustness,Iyengar2005robustDP} for detailed discussions.

One might be interested in a mixture of the above models, i.e., uncertainty sets determined by several KL divergence bounds. In the most general form, such an inner minimization problem can be formulated as. 
\begin{align}
\underset{\bx}{\text{minimize}}\qquad &   \bc^\T \bx  = \sum_{i=1}^{N_s} c_ix_i&\label{prob:general-uncertainty-model} \\
\text{subject to} \qquad &  \widehat{\bX}\ln\bx \geq \pmb{\alpha} &\nonumber\\
& (\bx^\T \ln\bx)\be - \widehat{\bY}\bx \leq \pmb{\beta}  & \nonumber\\
& \bx \in \Delta(|N_s|),\nonumber
\end{align}
 where $\ln\bx = (\ln x_1,\ldots,\ln x_{N_s})^\T$, 
 $\widehat{\bX} \in \bbR^{K\times N_s}_{+}$, $\pmb{\alpha} \in \bbR^{K}$ are parameters of the likelihood models,  and $\widehat{\bY}\in \bbR^{H\times N_s}_{+}$, $\pmb{\beta} \in \bbR^{H}$ are parameters of the relative entropy models. 
In general, it seems not possible to solve the above problem by bisection if $K+H\geq 2$, but we can prove that \eqref{prob:general-uncertainty-model} can be solved by interior-point in polynomial time.
\begin{proposition}
	\label{prop:sa-rect-complexity}
	Assume that \eqref{prob:general-uncertainty-model} satisfies the Slater condition, then a $\xi$-approximation of Problem \ref{prob:general-uncertainty-model} can be achieved with complexity $C(\xi) = \cO\left((4N_s+H+K+2)^{1/2}(4N_s+H+K)N_s^2 \ln \xi^{-1}\right)$. 
\end{proposition}
 \begin{proof}
 	By a change of variable, we write an equivalent problem
 	\begin{align}
 	\underset{\bx}{\text{minimize}}\qquad &   \bc^\T \bx &\label{prob:proof-general-uncertainty-model} \\
 	\text{subject to} \qquad &  \sum_i \widehat{X}_{ki}z_i \geq \alpha_k & \forall k\nonumber\\
 	& \sum_i y_i - \widehat{Y}_{hi} \leq \beta_h  & \forall h \nonumber\\
 	& z_i \leq \ln(x_i)  &  \forall i \nonumber\\
 	& y_i \geq x_i\ln x_i  &  \forall i \nonumber\\
 	& \sum_{i}x_i \geq 1-\epsilon &\nonumber \\
 	& -\sum_{i}x_i \geq -(1+\epsilon). &\nonumber 
 	\end{align}
 	With the following notes \cite{Nesterov1994interior}
 	\begin{itemize}
 		\item  $\Phi(x,z) = -\ln(\ln x-z) - \ln x$ is a 2-self-concordant barrier for the epigraph of $\{(x,z)|\ \ln(x)\geq z,x\geq 0\}$
 		\item  $\Gamma(x,y) = -\ln(y - x\ln x) - \ln x$ is a 2-self-concordant barrier for the epigraph of $\{(x,z)|\ x\ln(x)\leq y,x\geq 0\}$
 	\end{itemize}	
 	This allows us to construct a barrier function of the feasible set of Problem \ref{prob:proof-general-uncertainty-model}.
 	\begin{align}
 	\cF(\bx,\by,\bz) &= -\sum_k \ln\left(\sum_i\widehat{X}_{ki} - \alpha_k\right) - \sum_{h} \ln \left(-\sum_i (y_i + \widehat{Y}_{hi})+ \beta_h \right) + \sum_i \Phi(x_i,z_i)\nonumber \\
 	 &+ \sum_i \Gamma (x_i,y_i)  - \ln\left(\sum_i x_i +\epsilon -1 \right) - \ln\left(-\sum_i x_i -\epsilon  + 1 \right)\nonumber. 
 	\end{align}
 	We see that $\cF(\bx,\by,\bz)$ is a self-concordant \cite{Nesterov1994interior} with variable $4N_s+K+H +2$. The complexity of the path-following method associated with the aforementioned barrier is 
 	\[
 	\cO\left((4N_s+H+K+2)^{1/2}(4N_s+H+K)N_s^2 \ln \epsilon^{-1}\right).
 	\] 
 \end{proof}

Typically, $H,K \ll N_s$. As a result, the complexity can be  bounded by $\cO(N_s^{7/2}\ln\xi^{-1})$. 
	
\subsubsection{{$(s)$-rectangularity}}
In  this case, we need to  solve the inner problem 
$
\min_{\bq_s\in\cQ_s} \Bigg\{\sum_{a\in \cA}\exp\Big(r(a|s) +  \gamma \bbE_{s'}[V^*(s')] \Big)\Bigg\},
$ for any $s\in\cS$, to perform contraction iterations and compute optimal policies. 
The inner optimization is of the form
\begin{align}
\underset{\bx}{\text{minimize}}\qquad &   \sum_{a\in\cA} d_a\exp\left(\sum_{i=1}^{N_s} c_{ai}x_{ai}\right) &\label{prob:general-uncertainty-model-s-rect} \\
\text{subject to} \qquad &  \widehat{\bX}\ln\bx \geq \pmb{\alpha} &\nonumber\\
& (\bx^\T \ln\bx)\be - \widehat{\bY}\bx \leq \pmb{\beta}  & \nonumber\\
& \bx \in \Delta(N_s\times |\cA|),\nonumber
\end{align}
where $\widehat{\bX} \in \bbR^{K\times (|\cA|N_s)}_{+}$  and $\widehat{\bY}\in \bbR^{H\times (|\cA|N_s)}_{+}$ are parameter matrices of the likelihood models and entropy models, respectively. The proposition below shows that one can solve Problem \ref{prob:general-uncertainty-model-s-rect} in polynomial time.
\begin{proposition}
	\label{prop:s-rect-complexity}
	Assume that \eqref{prob:general-uncertainty-model-s-rect} satisfies the Slater condition, then a $\xi$-approximation of Problem \ref{prob:general-uncertainty-model-s-rect} can be achieved with complexity $$
	\begin{aligned}
		C(\xi) =& \cO\Big((4|\cA|N_s+4|\cA|+ H+K)^{1/2}(4|\cA|N_s+H+K)(|\cA|N_s)^2 \ln \xi^{-1}\Big).
	\end{aligned}$$	
\end{proposition}
\begin{proof}
	By a change of variable, we write an equivalent problem
\begin{align}
\underset{\bx,\by,\bz,\bt}{\text{minimize}}\qquad &   \sum_{a\in\cA} t_a   &\label{prob:proof-general-uncertainty-model-s-rect} \\
\text{subject to} \qquad &  \sum_{i=1}^{N_s}c_{ai}x_{ai} \leq \ln  t_a  & \forall a \nonumber\\
 &  \sum_{a,i} \widehat{X}_{kai}z_{ai} \geq \alpha_k & \forall k\nonumber\\
& \sum_{a,i} y_{ai} - \widehat{Y}_{hai} \leq \beta_h  & \forall h \nonumber\\
& z_{ai} \leq \ln(x_{ai})  &  \forall i \nonumber\\
& y_{ai} \geq x_{ai}\ln x_{ai}  &  \forall i \nonumber\\
& \sum_{i}x_{ai} \geq 1-\epsilon & \forall a\nonumber \\
& -\sum_{i}x_{ai} \geq -(1+\epsilon) &\forall a\nonumber 
\end{align}
A self-concordant barrier for the feasible set of Problem \ref{prob:proof-general-uncertainty-model} can be constructed as 
\begin{align}
\cF(\bx,\by,\bz,\bt) &= -\sum_k \ln\left(\sum_{a,i}\widehat{X}_{kai}z_{ai} - \alpha_k\right) - \sum_{h} \ln \left(-\sum_{a,i} (y_{ai} + \widehat{Y}_{hai})+ \beta_h \right) \nonumber \\
&+ \sum_{a,i} \Phi(x_{ai},z_{ai})+ \sum_{a,i} \Gamma (x_{ai},y_{ai})  + \sum_{a}\Phi\left(t_s, \sum_{i}c_{ai}x_{ai} \right) \nonumber \\
&- \sum_a \ln\left(\sum_i x_i +\epsilon -1 \right) - \sum_a\ln\left(-\sum_i x_i -\epsilon  + 1 \right)\nonumber. 
\end{align}
 $\cF(\bx,\by,\bz,\bt)$ is a self-concordant \cite{Nesterov1994interior} with variable $4|\cA|N_s + 4|\cA| +K+H$. The complexity of the path-following method associated with the aforementioned barrier is 
\[
\cO\left((4|\cA|N_s+4|\cA|+ H+K)^{1/2}(4|\cA|N_s+H+K)(|\cA|N_s)^2 \ln \epsilon^{-1}\right).
\] 
\end{proof} 

In cases $H,K\ll |\cA|N_s$ the  complexity is about $\cO((|\cA|N_s)^{7/2} \ln \xi^{-1} )$.

\end{document}